\documentclass[twocolumn,conference,letterpaper]{IEEEtran}

\usepackage{amssymb,amsmath,latexsym,amsfonts,amsthm,mathtools}
\usepackage{graphicx}
\usepackage{float}
\usepackage{caption}
\usepackage[font=footnotesize]{subcaption}
\usepackage{placeins}
\usepackage{textcomp}
\usepackage{xcolor}
\usepackage{booktabs}
\usepackage{cite}

\usepackage[hidelinks]{hyperref}
\hypersetup{hidelinks, breaklinks}
\usepackage[nameinlink]{cleveref}
\Crefformat{figure}{#2Fig.~#1#3}
\Crefmultiformat{figure}{Figs.~#2#1#3}{ and~#2#1#3}{, #2#1#3}{ and~#2#1#3}

\usepackage{tikz}
\usetikzlibrary{automata, shapes, arrows, calc, arrows.meta, fit, positioning}

\theoremstyle{plain}
\newtheorem{theorem}{Theorem}
\newtheorem{lemma}{Lemma}
\newtheorem{corollary}{Corollary}
\newtheorem*{problem*}{Problem}
\theoremstyle{remark}
\newtheorem{remark}{Remark}
\newtheorem{assumption}{Assumption}
\newtheorem{problem}{Problem}

\newtheorem{definition}{Definition}
\theoremstyle{definition}

\usepackage{mathrsfs}

\usepackage{newtxmath}

\DeclareMathOperator{\rank}{rank}
\DeclareMathOperator{\dist}{dist}
\DeclareMathOperator{\blkdiag}{blkdiag}
\DeclareMathOperator*{\argmax}{arg\,max}
\DeclareMathOperator*{\argmin}{arg\,min}
\newcommand{\R}{\mathbb{R}}
\newcommand{\W}{\mathcal{W}}
\newcommand{\Z}{\mathcal{Z}}
\newcommand{\T}{\mathcal{T}}
\newcommand{\K}{\mathcal{K}}
\newcommand{\Aset}{\mathcal{A}}
\newcommand{\B}{\mathbb{B}}
\newcommand{\M}{\mathbf{M}}
\newcommand{\A}{\mathbf{A}}
\newcommand{\e}{\mathbf{e}}
\newcommand{\w}{\mathbf{w}}
\newcommand{\z}{\mathbf{z}}
\newcommand{\one}{\mathbf{1}}
\newcommand{\zero}{\mathbf{0}}

\newcommand{\intset}{\operatorname{int}}
\newcommand{\bd}{\partial}

\newcommand{\SO}{\mathrm{SO}}

\definecolor{cNavy}{RGB}{20,45,70}
\definecolor{cTeal}{RGB}{42,132,130}
\definecolor{cGold}{RGB}{190,135,55}
\definecolor{cGray}{RGB}{235,238,240}

\IEEEoverridecommandlockouts

\begin{document}
	
	\title{Residual Wrench Certification and Margin-Aware Control Synthesis for Aerial Physical Interaction}
	
	\author{Abhimanyu Khadga, Abhinav Sinha,~\IEEEmembership{Senior Member,~IEEE}, and Shashi Ranjan Kumar,~\IEEEmembership{Senior Member,~IEEE}
		\thanks{A. Khadga and A. Sinha are with the Guidance, Autonomy, Learning, and Control for Intelligent Systems (GALACxIS) Lab, Department of Aerospace Engineering, University of Cincinnati, OH 45221, USA (e-mails: khadgaau@mail.uc.edu, abhinav.sinha@uc.edu). S. R. Kumar is with the Intelligent Systems and Control (ISaC) Lab, Department of Aerospace Engineering, Indian Institute of Technology Bombay, Powai, Mumbai 400076, India (email: srk@aero.iitb.ac.in).}
	}
	
	\maketitle
	\thispagestyle{empty}
	
	\begin{abstract}
		We develop a task-relative framework for certifying residual wrench authority after hover and contact loading in multirotors with bounded actuators. Using convex geometry, we derive signed margins for prescribed convex reserves, including Euclidean balls and weighted ellipsoids. We obtain computable reserve certificates from actuator-interiority bounds to support slack-maximizing allocation. To preserve the required reserve, we propose command projection onto a tightened feasible set. We then connect the available reserve to structured gain synthesis and certify a local tracking region that respects actuator limits. Within this region, we establish nominal exponential convergence and robust ultimate boundedness. For the prescribed task and morphology family, we show that the optimized octarotor attains a larger margin than the optimized hexarotor at equal total thrust.
        We evaluate the proposed framework in closed-loop simulations of sustained rigid-wall contact using a fixed-geometry octarotor and a variable-tilt quadrotor with matched installed thrust. The quadrotor's local certificate admits a higher normalized push limit for the prescribed task family. In tests beyond the realizability boundaries, we observe the predicted loss of authority through rotor-thrust limits for the octarotor and tilt-servo limits for the quadrotor.
	\end{abstract}
	

	\section{Introduction}
	Multirotors are attractive aerial-robotic platforms because they combine hovering capability with a compact mechanical layout. In a conventional planar arrangement, however, nearly parallel propeller axes couple translational acceleration to body attitude. Force--attitude coupling is often acceptable in free flight \cite{rashad2020review,hamandi2021taxonomy}, yet becomes restrictive when an aerial robot must preserve its pose while pressing, bracing, inspecting, or manipulating an object. Fully actuated and overactuated designs mitigate this limitation through tilted or steerable rotors, noncoplanar arrangements, coaxial vectoring, and other mechanisms that span a six-dimensional body wrench \cite{rajappa2015modeling,ryll2015overactuatedquadrotor,brescianini2016design,brescianini2018omni,rashad2017design,kotarski2021performance}. Safe interaction additionally depends on retaining corrective authority under sustained loading.
	
	Research on aerial interaction has advanced both vehicle design and constrained allocation. Tilted and omnidirectional platforms reduce force--attitude coupling \cite{allenspach2020design,ryll2022fasthex} and support inspection and manipulation \cite{park2018odar,ryll2019flyingend,bodie2019omnidirectional,bodie2020active,hwang2024pushing,hui2024passive,veenstra2026sixd,lin2025float}. Allocation methods account for saturation and redundancy \cite{allenspach2020design,hwang2024pushing}, with recent formulations incorporating actuator and power dynamics \cite{cuniato2024allocation}. Safety filters also preserve motor-thrust limits during uncertain physical interaction \cite{byun2025safety}.
	
	Prescribed task-wrench sets have been studied for cable-driven robots \cite{bouchard2010prescribed}, and feasible wrench polytopes support legged-robot disturbance analysis \cite{orsolino2018wrench,prieto2022feasible}. Signed capacity measures also appear in redundantly constrained cable systems \cite{harms2024tetherbot}. For aerial interaction, force-polytope methods use zero-moment force margins for tilt-angle selection \cite{piccina2026force}. These state-of-the-art approaches do not connect a uniform reserve over coupled force--moment task families to the joint synthesis of feedback gains and a certified local tracking region.
	
	Closing this gap is crucial for future aerial vehicles expected to sustain autonomous physical interaction. A feasible contact wrench can leave negligible authority for pose correction near a surface. Certification must cover an entire task family while accommodating directional reserve requirements. These certificates must also distinguish exhaustion of a prescribed reserve from actual infeasibility. Moreover, redundancy complicates computation in the sense that actuator clearance depends on the allocation, while exact wrench geometry can require extensive facet enumeration. The resulting control problem couples gain selection to the size of a tracking region that remains actuator-feasible under bounded disturbances.

	To address these limitations, we develop a task-relative certification-and-synthesis framework that converts the available wrench reserve into an actuator-feasible tracking guarantee: \emph{(i)} We derive convex-geometric certificates that quantify signed margins for prescribed reserves and uniform Euclidean margins over entire task families. \emph{(ii)} We obtain actuator-interiority bounds that avoid facet enumeration and strengthen these lower certificates through slack-maximizing redundant allocation. \emph{(iii)} We propose margin-preserving command projection and use the available reserve to synthesize structured gains with a local invariant tracking ellipsoid. Within this ellipsoid, we establish nominal exponential convergence and robust ultimate boundedness under the stated assumptions. \emph{(iv)} We assess morphology and allocation effects at equal installed thrust. Closed-loop contact simulations distinguish the prescribed reserve threshold from actual loss of realizability and test the predicted limiting constraints through preregistered beyond-boundary experiments.
	
	\emph{Contact-persistent} denotes persistence of actuator-side wrench feasibility during sustained loading. The certificates support safety-aware allocation and local tracking under the stated assumptions. Physical persistence additionally requires unilateral, frictional, and contact-moment conditions. Those constraints may be included by intersecting the actuator-generated wrench set with a contact-admissible set.
	
	\textit{Notation:} Unsubscripted vector norms are Euclidean, and matrix norms are induced unless stated otherwise. The symbols $\mathbf e_i$, $\zero_n$, $\one_n$, and $\mathbf I_n$ denote the $i$th standard basis vector, the zero vector or matrix, the all-ones vector, and the identity matrix of the indicated dimension. For a set $\mathcal S\subset\R^n$, $\intset(\mathcal S)$ and $\partial\mathcal S$ denote its interior and boundary, respectively, while $\B_{\epsilon}(\mathbf x)=\{\mathbf y:\|\mathbf y-\mathbf x\|\leq\epsilon\}$ is the closed Euclidean ball. We use $\dist(\mathbf x,\mathcal S)=\inf_{\mathbf y\in\mathcal S}\|\mathbf x-\mathbf y\|$ and $\dist(\mathcal S_1,\mathcal S_2)=\inf_{\mathbf x\in\mathcal S_1,\mathbf y\in\mathcal S_2}\|\mathbf x-\mathbf y\|$. For symmetric matrices, $\mathbf X\preceq\mathbf Y$ means that $\mathbf Y-\mathbf X$ is positive semidefinite.
	
	\section{Modeling and Problem Formulation}\label{sec:model}
	The vehicle state is described on $\R^{3}\times\SO\left(3\right)$. Position and inertial velocity are denoted by $\mathbf{p},\mathbf{v}\in\R^{3}$, while $\mathbf{R}\in\SO\left(3\right)$ maps body coordinates to the inertial frame and $\boldsymbol{\omega}\in\R^{3}$ is the body angular rate. The rigid-body equations are given as
	\begin{subequations}
		\label{eq:p_dot}
		\label{eq:R_dot}
		\begin{alignat}{2}
			\dot{\mathbf{p}} =& \mathbf{v}, &~~ m\dot{\mathbf{v}} =& -mg\e_{3}+\mathbf{R}\mathbf{f}+\mathbf{d}_{f},
			\label{eq:v_dot}\\
			\dot{\mathbf{R}} =& \mathbf{R}S\left(\boldsymbol{\omega}\right), &~~ \mathbf{J}\dot{\boldsymbol{\omega}} =& -\boldsymbol{\omega}\times\mathbf{J}\boldsymbol{\omega}+\boldsymbol{\tau}+\mathbf{d}_{\tau},
			\label{eq:omega_dot}
		\end{alignat}
	\end{subequations}
	where $m>0$, $\mathbf{J}=\mathbf{J}^{\top}>0$, and $g>0$ are the mass, inertia, and gravitational constant. The vectors $\mathbf{d}_{f}$ and $\mathbf{d}_{\tau}$ collect unmodeled force and moment effects, and $S\left(\mathbf{x}\right)\mathbf{y}=\mathbf{x}\times\mathbf{y}$. The body-frame actuation wrench is $\w =\begin{bmatrix}\mathbf{f}^{\top}&\boldsymbol{\tau}^{\top}\end{bmatrix}^{\top}\in\R^{6}$.
	
	Suppose that $N$ thrust modules are mounted at body-frame locations $\mathbf{r}_{i}$. Module $i$ produces the nonnegative magnitude $\lambda_{i}$ along the unit vector $\mathbf{b}_{i}\left(\boldsymbol{\alpha}_{i}\right)$, where $\boldsymbol{\alpha}_{i}$ represents fixed or commanded orientation variables. The force and moment contributions are
	\begin{align}
		\mathbf{f}_{i} =& \lambda_{i}\mathbf{b}_{i}\left(\boldsymbol{\alpha}_{i}\right),~~
		\boldsymbol{\tau}_{i} = \mathbf{r}_{i}\times\lambda_{i}\mathbf{b}_{i}\left(\boldsymbol{\alpha}_{i}\right)+\sigma_{i}c_{i}\lambda_{i}\mathbf{b}_{i}\left(\boldsymbol{\alpha}_{i}\right),
	\end{align}
	where $c_{i}\geq0$ models the reaction-moment coefficient and $\sigma_{i}\in\{-1,1\}$ specifies the spin sense. Summing the module contributions leads to
	\begin{align}
		\w =& \sum_{i=1}^{N}
		\begin{bmatrix}
			\mathbf{b}_{i}\left(\boldsymbol{\alpha}_{i}\right)\\
			S\left(\mathbf{r}_{i}\right)\mathbf{b}_{i}\left(\boldsymbol{\alpha}_{i}\right)
			+\sigma_{i}c_{i}\mathbf{b}_{i}\left(\boldsymbol{\alpha}_{i}\right)
		\end{bmatrix}\lambda_{i}.
		\label{eq:wrench_sum_short}
	\end{align}
	With $\boldsymbol{\lambda}=\begin{bmatrix}\lambda_{1}&\cdots&\lambda_{N}\end{bmatrix}^{\top}$ and $\boldsymbol{\alpha}$ collecting the morphology variables, \eqref{eq:wrench_sum_short} becomes $\w=\M\left(\boldsymbol{\alpha}\right)\boldsymbol{\lambda}$. The admissible thrust box and morphology domain are
	$\Lambda=\{\boldsymbol{\lambda}:\boldsymbol{\lambda}_{\min}\leq\boldsymbol{\lambda}\leq\boldsymbol{\lambda}_{\max}\}$ and $\boldsymbol{\alpha}\in\Aset$, respectively, where $\Aset$ is compact. For a selected morphology, the attainable wrench set is $\W(\boldsymbol{\alpha}):=\{\w:\w=\M(\boldsymbol{\alpha})\boldsymbol{\lambda},~\boldsymbol{\lambda}\in\Lambda\}$.
	Force and moment coordinates have different units, so every Euclidean wrench distance must be evaluated after a fixed scaling. We choose $F_0,\ell_0>0$ and define
	\begin{align}
		\mathbf{D}_{w}:=&\blkdiag\!\left(F_{0}^{-1}\mathbf{I}_{3},
		(\ell_{0}F_{0})^{-1}\mathbf{I}_{3}\right),
		~~ \z=\mathbf D_w\w,
		\label{eq:control_wrench_scaling}\\
		\A(\boldsymbol{\alpha}):=&\mathbf D_w\M(\boldsymbol{\alpha}),~
		\Z(\boldsymbol{\alpha}):=\mathbf D_w\W(\boldsymbol{\alpha})=\{\A(\boldsymbol\alpha)\boldsymbol\lambda:\boldsymbol\lambda\in\Lambda\}.
		\label{eq:normalized_feasible_set_control}
	\end{align}
	Note that all balls, distances, reserve margins, and induced allocation norms below are taken in this dimensionless wrench coordinate. Since $\mathbf D_w$ is nonsingular, membership, interiority, and rank are preserved by the transformation. A margin of $\varepsilon$ in these coordinates therefore guarantees actuator feasibility for any normalized wrench perturbation within the corresponding reserve ball; the force and moment scaling factors are $\varepsilon F_0$ and $\varepsilon\ell_0F_0$, respectively.
	\begin{definition}
		\label{def:constrained_full_actuation}
		At an operating wrench $\w_{0}$, with $\z_0=\mathbf D_w\w_0$, constrained full actuation holds when $\z_0\in\intset(\Z(\boldsymbol\alpha))$; equivalently, $\w_0\in\intset(\W(\boldsymbol\alpha))$.
	\end{definition}
	\Cref{def:constrained_full_actuation} is operating-point dependent. Rank six describes the span of the allocation map, whereas constrained full actuation also depends on the finite actuator box and on how close $\z_{0}$ is to the resulting normalized boundary.
	\begin{problem}[Task reserve and feasible tracking]
		\label{prob:margin_control_synthesis}
		Given $\Lambda$, $\Aset$, and a nonempty compact family $\T_z\subset\R^6$ of normalized hover and interaction wrenches, determine which morphologies admit every task. For each feasible morphology and a prescribed compact convex reserve set $\K\subset\R^6$ with $\zero_6\in\intset(\K)$, characterize the largest $\epsilon\geq0$ satisfying $\z+\epsilon\K\subseteq\Z(\boldsymbol{\alpha})$, $\forall\z\in\T_z$. For a Euclidean reserve, derive computable lower certificates and use the available authority to select actuator-feasible commands and structured feedback gains. Certify a local tracking region with nominal exponential convergence and robust ultimate boundedness under quantified disturbance bounds.
	\end{problem}
	
	\section{Main Results}
	\label{sec:theory}
	For a task wrench $\z_{\mathrm{t}}\in\T_z$, translating the feasible set by the assigned demand produces the remaining authority
	\begin{align}
		\Z_{\mathrm{r}}\left(\boldsymbol{\alpha},\z_{\mathrm{t}}\right)
		:=&\left\{\Delta\z:\z_{\mathrm{t}}+\Delta\z\in\Z\left(\boldsymbol{\alpha}\right)\right\}.
		\label{eq:residual_set}
	\end{align}
	Every element of $\Z_{\mathrm{r}}$ is a normalized additional wrench that may be generated without violating the actuator limits after the task load has been committed.
	\begin{definition}
		\label{def:contact_persistent}
		For $\epsilon>0$, the morphology is contact-persistently fully actuated at $\z_{\mathrm{t}}$ with Euclidean margin $\epsilon$ when
		$\B_{\epsilon}\left(\zero_{6}\right)\subseteq\Z_{\mathrm{r}}\left(\boldsymbol{\alpha},\z_{\mathrm{t}}\right)$.
	\end{definition}
	\begin{remark}
		\label{rem:contact_scope}
		Only the actuator limits are covered by \eqref{eq:normalized_feasible_set_control}. Let $\W_{\mathrm{ct}}^{\mathrm{act}}$ be the convex set of total body actuation wrenches compatible with admissible contact reactions after transformation to the same body point and inclusion of the required equilibrium terms. With $\Z_{\mathrm{ct}}^{\mathrm{act}}:=\mathbf{D}_{w}\W_{\mathrm{ct}}^{\mathrm{act}}$, define $\Z_{\mathrm{adm}}(\boldsymbol{\alpha}):=\Z(\boldsymbol{\alpha})\cap\Z_{\mathrm{ct}}^{\mathrm{act}}$.
		The geometric results below are stated for the actuator-generated set $\Z(\boldsymbol\alpha)$. If unilateral, frictional, and contact-moment constraints can be represented in the same normalized body-wrench coordinates by a convex set $\Z_{\mathrm{ct}}^{\mathrm{act}}$, then $\Z$ may be replaced by $\Z_{\mathrm{adm}}$ in the containment, support-function, signed-margin, and uniform-reserve results, provided the intersection is full dimensional at the operating point. In the sequel, the actuator-interiority certificate and the slack-maximizing allocation use only clearance from the thrust bounds; they account for contact restrictions only when the corresponding contact inequalities are included directly in the allocation problem.
	\end{remark}
	\begin{assumption}
		\label{ass:rank}
		The selected morphology satisfies $\rank(\A(\boldsymbol\alpha))=6$, and every actuator interval has nonzero width.
	\end{assumption}
	\begin{lemma}
		\label{lem:geometry}
		Under \Cref{ass:rank}, $\Z\left(\boldsymbol{\alpha}\right)$ is a compact, convex, six-dimensional polytope.
	\end{lemma}
	\begin{proof}
		The thrust box $\Lambda$ in \eqref{eq:normalized_feasible_set_control} is a compact convex polytope, so its linear image under $\A(\boldsymbol{\alpha})$ has the same properties. By \Cref{ass:rank}, the box has nonempty interior and $\A(\boldsymbol{\alpha})$ maps onto $\R^{6}$. Hence, $\Z(\boldsymbol{\alpha})$ has nonempty interior in $\R^{6}$.
	\end{proof}
	\begin{remark}
		\label{prop:rank_not_enough}
		Actuator bounds impose feasibility and reserve requirements beyond structural rank. Compactness in \Cref{lem:geometry} ensures that some $\z_{\mathrm{out}}\in\R^{6}$ lies outside $\Z(\boldsymbol{\alpha})$. At any $\z_{\mathrm{b}}\in\bd\Z(\boldsymbol{\alpha})$, every neighborhood intersects the complement of $\Z(\boldsymbol{\alpha})$. By \eqref{eq:residual_set}, such a boundary wrench admits no positive-radius residual ball.
	\end{remark}
	\begin{lemma}
		\label{prop:euclidean_reserve}
		For a feasible task wrench, the largest Euclidean residual radius is $\epsilon^{\star}(\boldsymbol{\alpha},\z_{\mathrm{t}})=\dist(\z_{\mathrm{t}},\bd\Z(\boldsymbol{\alpha}))$.
		The radius is positive if and only if $\z_{\mathrm{t}}\in\intset(\Z(\boldsymbol{\alpha}))$.
	\end{lemma}
	\begin{proof}
		The inclusion in \Cref{def:contact_persistent}, translated using \eqref{eq:residual_set}, is equivalent to $\B_{\epsilon}(\z_{\mathrm{t}})\subseteq\Z(\boldsymbol{\alpha})$. Positive-radius containment characterizes interior task wrenches. By \Cref{lem:geometry}, $\Z(\boldsymbol{\alpha})$ is closed, so the ball with radius $\epsilon^{\star}(\boldsymbol{\alpha},\z_{\mathrm{t}})$ is contained in $\Z(\boldsymbol{\alpha})$, whereas every larger ball intersects the complement of the feasible set.
	\end{proof}
	\Cref{prop:euclidean_reserve} uses an isotropic ball. A controller may instead need a direction-dependent reserve. Suppose
	\begin{align}
		\Z\left(\boldsymbol{\alpha}\right)
		=&\left\{\z:\mathbf{a}_{j}^{\top}\z\leq b_{j},~j=1,\ldots,q\right\},
		\label{eq:h_rep}
	\end{align}
	where every half-space normal satisfies $\mathbf a_j\neq\zero_6$. Let $\K\subset\R^{6}$ be compact and convex with $\zero_{6}\in\intset\left(\K\right)$, and define the support function $h_{\K}\left(\mathbf{a}\right)=\sup_{\mathbf{k}\in\K}\mathbf{a}^{\top}\mathbf{k}$.
	\begin{theorem}
		\label{thm:general_reserve}
		Define
		\begin{align}
			\varphi_{\Z,\K}\left(\boldsymbol{\alpha},\z_{0}\right)
			:=&\min_{j}
			\frac{b_{j}-\mathbf{a}_{j}^{\top}\z_{0}}
			{h_{\K}\left(\mathbf{a}_{j}\right)}.
			\label{eq:general_reserve_margin}
		\end{align}
		Since $\zero_{6}\in\intset(\K)$, every nonzero half-space normal satisfies $h_{\K}(\mathbf a_j)>0$. Hence, $\varphi_{\Z,\K}$ is positive at interior points, zero on the boundary, and negative outside $\Z$. For every feasible $\z_{0}$, the largest nonnegative scale satisfying $\z_{0}+\epsilon\K\subseteq\Z(\boldsymbol{\alpha})$ is $\epsilon_{\K}^{\star}=\varphi_{\Z,\K}(\boldsymbol{\alpha},\z_{0})$.
	\end{theorem}
	\begin{proof}
		For $\epsilon\geq0$, the support-function definition and \eqref{eq:h_rep} make $\z_{0}+\epsilon\K\subseteq\Z(\boldsymbol{\alpha})$ equivalent to $\mathbf{a}_{j}^{\top}\z_{0}+\epsilon h_{\K}(\mathbf{a}_{j})\leq b_{j}$ for all $j$. The denominators $h_{\K}(\mathbf{a}_{j})$ are positive because $\zero_{6}\in\intset(\K)$. Dividing each facet slack by the corresponding positive support value and taking the minimum gives \eqref{eq:general_reserve_margin}. The facet slacks in \eqref{eq:h_rep} also establish the stated sign characterization.
	\end{proof}
	\begin{corollary}
		\label{cor:signed_margin_ball}
		For $\K=\B_{1}(\zero_{6})$, the support function is $h_{\K}(\mathbf{a})=\|\mathbf{a}\|$. Thus, \Cref{thm:general_reserve} specializes to
		\begin{align}
			\varphi_{\Z}\left(\boldsymbol{\alpha},\z\right)
			=&\min_{j}\frac{b_{j}-\mathbf{a}_{j}^{\top}\z}{\left\|\mathbf{a}_{j}\right\|},
			\label{eq:signed_margin_z}
		\end{align}
		where strictly positive, zero, and negative values identify an interior wrench, a boundary wrench, and an infeasible wrench, respectively. At an infeasible wrench, the negative magnitude is a normalized half-space violation and need not equal the Euclidean distance to $\Z$; comparisons of negative values, thus, use the same exact facet representation.
	\end{corollary}
	\begin{corollary}
		\label{cor:ellipsoidal_reserve}
		For a feasible $\z_0$, let $\mathbf{Q}_{w}=\mathbf{Q}_{w}^{\top}>0$ and
		$\K_{\mathbf{Q}_{w}}=\{\mathbf{k}:\mathbf{k}^{\top}\mathbf{Q}_{w}\mathbf{k}\leq1\}$. The support function is $h_{\K_{\mathbf{Q}_{w}}}(\mathbf{a})=\sqrt{\mathbf{a}^{\top}\mathbf{Q}_{w}^{-1}\mathbf{a}}$, so \Cref{thm:general_reserve} yields the reserve $\epsilon_{\mathbf{Q}_{w}}^{\star}=\min_{j}\frac{b_{j}-\mathbf{a}_{j}^{\top}\z_{0}}{\sqrt{\mathbf{a}_{j}^{\top}\mathbf{Q}_{w}^{-1}\mathbf{a}_{j}}}$.
		The matrix $\mathbf{Q}_{w}$ specifies unequal reserve requirements across force and moment directions.
	\end{corollary}
	Explicit facet construction can become burdensome as the number of actuators grows. A lower bound can instead be obtained from an interior actuator allocation, as shown next.
	\begin{theorem}
		\label{thm:certificate}
		Let $\mathbf{R}_{A}\in\R^{N\times6}$ satisfy $\A(\boldsymbol{\alpha})\mathbf{R}_{A}=\mathbf I_6$, and define $\|\mathbf{R}_{A}\|_{2\rightarrow\infty}:=\sup_{\mathbf v\neq\zero_6}\|\mathbf{R}_{A}\mathbf v\|_{\infty}/\|\mathbf v\|_2$. Suppose $\boldsymbol{\lambda}_{0}\in\intset(\Lambda)$ realizes $\z_{\mathrm{t}}$. Define the actuator clearance $\delta_{\infty}(\boldsymbol{\lambda}_{0}):=\min_{i}\{\lambda_{0,i}-\lambda_{i,\min},~\lambda_{i,\max}-\lambda_{0,i}\}$.
		Then, the certified Euclidean reserve is
		\begin{align}
			\underline{\epsilon}
			=&\frac{\delta_{\infty}\left(\boldsymbol{\lambda}_{0}\right)}
			{\left\|\mathbf{R}_{A}\right\|_{2\rightarrow\infty}}.
			\label{eq:certificate_margin}
		\end{align}
	\end{theorem}
	\begin{proof}
		For $\|\Delta\z\|\leq\underline{\epsilon}$, choose $\Delta\boldsymbol{\lambda}=\mathbf{R}_{A}\Delta\z$. By \eqref{eq:certificate_margin}, $\|\Delta\boldsymbol{\lambda}\|_{\infty}\leq\|\mathbf{R}_{A}\|_{2\rightarrow\infty}\|\Delta\z\|\leq\delta_{\infty}(\boldsymbol{\lambda}_{0})$.
		The clearance in \Cref{thm:certificate} ensures $\boldsymbol{\lambda}_{0}+\Delta\boldsymbol{\lambda}\in\Lambda$. The right-inverse identity implies $\A(\boldsymbol{\alpha})(\boldsymbol{\lambda}_{0}+\Delta\boldsymbol{\lambda})=\z_{\mathrm{t}}+\Delta\z$, so \eqref{eq:normalized_feasible_set_control} places the entire perturbation ball inside $\Z(\boldsymbol{\alpha})$, as required by \Cref{def:contact_persistent}.
	\end{proof}
	\begin{corollary}
		\label{cor:slack_max}
		For a fixed task and morphology, consider
		\begin{align}
			(\boldsymbol{\lambda}^{\star},\delta^{\star})
			\in&\argmax_{\boldsymbol{\lambda},\,\delta\geq0}\delta
			~~\mathrm{s.t.}~~\A(\boldsymbol{\alpha})\boldsymbol{\lambda}=\z_{\mathrm{t}},
			\label{eq:slack_lp}\\
			&\boldsymbol{\lambda}_{\min}+\delta\one_{N}\leq\boldsymbol{\lambda}
			\leq\boldsymbol{\lambda}_{\max}-\delta\one_{N}.\nonumber
		\end{align}
		For a feasible linear program, interiority of the optimized allocation is equivalent to $\delta^{\star}>0$. For a selected right inverse, the certified reserve is $\underline{\epsilon}^{\star}=\delta^{\star}/\|\mathbf{R}_{A}\|_{2\rightarrow\infty}$.
	\end{corollary}
	\begin{proof}
		The constraints in \eqref{eq:slack_lp} imply $\delta_{\infty}(\boldsymbol{\lambda}^{\star})\geq\delta^{\star}$, with clearance defined in \Cref{thm:certificate}. Equality holds at an optimum, since a larger clearance would permit a larger feasible $\delta$. Hence, $\boldsymbol{\lambda}^{\star}\in\intset(\Lambda)$ if and only if $\delta^{\star}>0$. For $\delta^{\star}>0$, the stated reserve follows from \Cref{thm:certificate} and \eqref{eq:certificate_margin}. When $\delta^{\star}=0$, feasibility in \eqref{eq:slack_lp} certifies the zero-radius case.
	\end{proof}
	For the task family $\T_z$, define the signed uniform margin
	\begin{align}
		\varphi_{\T_z}^{\star}\left(\boldsymbol{\alpha}\right)
		:=&\inf_{\z\in\T_z}\varphi_{\Z}\left(\boldsymbol{\alpha},\z\right).
		\label{eq:signed_uniform_margin}
	\end{align}
	\begin{theorem}
		\label{thm:uniform}
		A common positive-radius Euclidean reserve exists for every $\z\in\T_z$ if and only if
		$\varphi_{\T_z}^{\star}\left(\boldsymbol{\alpha}\right)>0$, equivalently,
		$\T_z\subset\intset\left(\Z\left(\boldsymbol{\alpha}\right)\right)$. If $\varphi_{\T_z}^{\star}\geq0$, then $\T_z\subseteq\Z(\boldsymbol\alpha)$ and the largest uniform radius $\epsilon\geq0$ satisfying
		$\z+\B_{\epsilon}(\zero_6)\subseteq\Z(\boldsymbol\alpha)$ for every $\z\in\T_z$ is $\epsilon_{\T_z}^{\star}(\boldsymbol{\alpha}):=\varphi_{\T_z}^{\star}(\boldsymbol{\alpha})=\dist(\T_z,\bd\Z(\boldsymbol{\alpha}))$.
		If $\varphi_{\T_z}^{\star}<0$, some task is infeasible and no nonnegative reserve satisfies containment; $\max\{0,\varphi_{\T_z}^{\star}\}$ is then only a numerical diagnostic.
	\end{theorem}
	\begin{proof}
		The pointwise margin in \eqref{eq:signed_margin_z} is a finite minimum of continuous affine functions of $\z$. Compactness of $\T_z$ in \Cref{prob:margin_control_synthesis} makes the infimum in \eqref{eq:signed_uniform_margin} attainable. The sign characterization in \Cref{cor:signed_margin_ball} then establishes the equivalence between a positive infimum and $\T_z\subset\intset(\Z(\boldsymbol{\alpha}))$. For feasible task wrenches, \Cref{prop:euclidean_reserve} identifies the pointwise reserve with the distance to $\bd\Z(\boldsymbol{\alpha})$. Taking the infimum over $\T_z$ yields $\dist(\T_z,\bd\Z(\boldsymbol{\alpha}))$, including the zero-margin case. A negative infimum is attained at an infeasible wrench by \Cref{cor:signed_margin_ball}, so no nonnegative reserve satisfies containment.
	\end{proof}
	Let
	$\mathbf{x}_{e}=\begin{bmatrix}\mathbf{q}^{\top}&\boldsymbol{\nu}^{\top}\end{bmatrix}^{\top}\in\R^{12}$ collect local pose and velocity errors, where $\mathbf q,\boldsymbol\nu\in\R^6$. Position, attitude, linear velocity, and angular velocity carry different units, so we fix a constant positive diagonal matrix $\mathbf{S}_{x}\in\R^{12\times12}$ before synthesis and define the dimensionless state $\bar{\mathbf{x}}_{e}=\mathbf{S}_{x}\mathbf{x}_{e}$. The controller uses the normalized wrench coordinates defined in \eqref{eq:control_wrench_scaling}.
	Assume that the reference trajectory is twice continuously differentiable \cite{9924233,10251969}, the morphology is fixed over the local synthesis interval, and the nominal rigid-body terms used in the cancellation are available within the selected coordinate chart. With $\mathbf{D}=\blkdiag(m\mathbf{I}_{3},\mathbf{J})>0$, cancellation of known reference, gravity, and rigid-body coupling terms leads to the local normalized error model
	\begin{align}
		\dot{\bar{\mathbf{x}}}_{e}
		=&\bar{\mathbf{A}}_{0}\bar{\mathbf{x}}_{e}
		+\bar{\mathbf{B}}_{0}\left(\z_{\mathrm{s}}+\mathbf{d}_{e}\right),
		\label{eq:normalized_local_error}\\
		\bar{\mathbf{A}}_{0}:=&\mathbf{S}_{x}
		\begin{bmatrix}
			\zero_{6\times6}&\mathbf{I}_{6}\\
			\zero_{6\times6}&\zero_{6\times6}
		\end{bmatrix}\mathbf{S}_{x}^{-1},~~
		\bar{\mathbf{B}}_{0}:=\mathbf{S}_{x}
		\begin{bmatrix}
			\zero_{6\times6}\\
			\mathbf{D}^{-1}
		\end{bmatrix}\mathbf{D}_{w}^{-1}.
		\nonumber
	\end{align}
	Here $\mathbf{d}_{e}\in\R^6$ is a normalized, locally essentially bounded matched disturbance that remains after feedforward compensation. Separately, $\z_{\mathrm{a}}\in\R^6$ is a bounded feedforward command whose modeled effect is included in the nominal cancellation; any uncompensated remainder belongs to $\mathbf d_e$. Thus, $\mathbf{d}_{e}$ affects the error dynamics, whereas $\z_{\mathrm{a}}$ consumes actuator reserve. We assume
	\begin{align}
		\|\mathbf{d}_{e}(t)\|\leq&\bar d_{e},
		~~
		\|\z_{\mathrm{a}}(t)\|\leq\bar z_{\mathrm{a}},
		\label{eq:control_uncertainty_bounds}
	\end{align}
	and use the uniform Euclidean margin $\epsilon_{\T_z}^{\star}$ from \Cref{thm:uniform}. The available feedback budget is $\epsilon_{c}:=\epsilon_{\T_{z}}^{\star}-\bar z_{\mathrm{a}}>0$.
	Consider a continuously parameterized gain family $\mathbf{K}(\boldsymbol{\theta})$ over a compact parameter set $\Theta$, and the normalized feedback law
	\begin{align}
		\z_{\mathrm{s}}=&-\mathbf{K}(\boldsymbol{\theta})\bar{\mathbf{x}}_{e},
		~
		\bar{\mathbf{A}}_{\boldsymbol{\theta}}:=
		\bar{\mathbf{A}}_{0}-\bar{\mathbf{B}}_{0}\mathbf{K}(\boldsymbol{\theta}).
		\label{eq:structured_feedback}
	\end{align}
	For a prescribed $\beta>0$, define the nonempty candidate set
	$\Theta_{\beta}:=\{\boldsymbol\theta\in\Theta:\max\Re\lambda(\bar{\mathbf A}_{\boldsymbol\theta})\leq-\beta\}$ and assume $\mathbf K(\boldsymbol\theta)\neq\zero_{6\times12}$ on $\Theta_\beta$. Choose $\mathbf{Q}_{L}=\mathbf{Q}_{L}^{\top}>0$. For every $\boldsymbol\theta\in\Theta_\beta$, let $\mathbf{P}_{\boldsymbol{\theta}}>0$ be the unique solution of
	\begin{align}
		\bar{\mathbf{A}}_{\boldsymbol{\theta}}^{\top}\mathbf{P}_{\boldsymbol{\theta}}
		+\mathbf{P}_{\boldsymbol{\theta}}\bar{\mathbf{A}}_{\boldsymbol{\theta}}
		=&-\mathbf{Q}_{L}.
		\label{eq:structured_lyapunov}
	\end{align}
	Define
	\begin{align}
		\kappa_{\boldsymbol{\theta}}:=&
		\left\|\mathbf{K}(\boldsymbol{\theta})
		\mathbf{P}_{\boldsymbol{\theta}}^{-1/2}\right\|_{2},
		~~
		c_{\boldsymbol{\theta}}:=
		\left(\frac{\epsilon_{c}}{\kappa_{\boldsymbol{\theta}}}\right)^{2}.
		\label{eq:structured_level}
	\end{align}
	The corresponding tracking ellipsoid is $\mathcal{E}_{\boldsymbol{\theta}}:=\{\bar{\mathbf{x}}_{e}:\bar{\mathbf{x}}_{e}^{\top}\mathbf{P}_{\boldsymbol{\theta}}\bar{\mathbf{x}}_{e}\leq c_{\boldsymbol{\theta}}\}$. The local-coordinate domain is represented by the prescribed dimensionless ellipsoid
	$\mathcal{E}_{\mathrm{loc}}=
	\{\bar{\mathbf{x}}_{e}:\bar{\mathbf{x}}_{e}^{\top}
	\mathbf{X}_{\mathrm{loc}}^{-1}\bar{\mathbf{x}}_{e}\leq1\}$, where $\mathbf X_{\mathrm{loc}}=\mathbf X_{\mathrm{loc}}^{\top}>0$.
	
	To address the controller-design objective in \Cref{prob:margin_control_synthesis}, we maximize the certified Lyapunov level $c_{\boldsymbol{\theta}}$ from \eqref{eq:structured_level} over gains satisfying the prescribed pole-abscissa and local-coordinate constraints:
	\begin{equation}
		\boldsymbol{\theta}^{\star}\in\argmax_{\boldsymbol{\theta}\in\Theta_{\beta}}
		c_{\boldsymbol{\theta}}
		~~\mathrm{s.t.}~~c_{\boldsymbol{\theta}}\mathbf{P}_{\boldsymbol{\theta}}^{-1}
		\preceq\mathbf{X}_{\mathrm{loc}}.
		\label{eq:structured_synthesis}
	\end{equation}
	The feasible set in \eqref{eq:structured_synthesis} is a closed subset of the compact set $\Theta_\beta$. The Lyapunov solution in \eqref{eq:structured_lyapunov} varies continuously with $\boldsymbol{\theta}$, so the objective in \eqref{eq:structured_level} is continuous. An optimizer exists whenever the feasible set is nonempty. The generally nonconvex optimization is low dimensional for structured gain families and can be solved by a reproducible grid-refinement or derivative-free search, e.g., \cite{conn2009derivativefree}. Every candidate must be tested against both the pole-abscissa and local-chart constraints; the local-chart constraint may be omitted only when verified to be inactive. 
	\begin{theorem}
		\label{thm:margin_tracking}
		Let $\boldsymbol{\theta}^{\star}$ solve \eqref{eq:structured_synthesis}, and abbreviate
		$\mathbf{K}^{\star}=\mathbf{K}(\boldsymbol{\theta}^{\star})$,
		$\mathbf{P}^{\star}=\mathbf{P}_{\boldsymbol{\theta}^{\star}}$,
		$c^{\star}=c_{\boldsymbol{\theta}^{\star}}$. Suppose
		$\z_{\mathrm{t}}(t)\in\T_{z}$, the disturbance bound in \eqref{eq:control_uncertainty_bounds} holds, and the local model remains valid in $\mathcal{E}_{\boldsymbol{\theta}^{\star}}$. Define $V:=\bar{\mathbf x}_e^{\top}\mathbf P^{\star}\bar{\mathbf x}_e$, $a^{\star}:=\frac{\lambda_{\min}(\mathbf{Q}_{L})}
		{\lambda_{\max}(\mathbf{P}^{\star})}$,  $g^{\star}:=\left\|(\mathbf{P}^{\star})^{1/2}
		\bar{\mathbf{B}}_{0}\right\|_{2}$, and $\chi^{\star}:=2g^{\star}\bar d_{e}/(a^{\star}\sqrt{c^{\star}})$.
		Then, for every initial condition $\bar{\mathbf x}_e(0)\in\mathcal E_{\boldsymbol\theta^\star}$:
		\begin{enumerate}
			\item if $\mathbf{d}_{e}\equiv\zero_{6}$, the ellipsoid
			$\mathcal{E}_{\boldsymbol{\theta}^{\star}}$ is positively invariant and
			\begin{align}
				\|\bar{\mathbf{x}}_{e}(t)\|
				\leq&\sqrt{\frac{\lambda_{\max}(\mathbf{P}^{\star})}
					{\lambda_{\min}(\mathbf{P}^{\star})}}
				\exp\!\left(-\frac{a^{\star}}{2}t\right)
				\|\bar{\mathbf{x}}_{e}(0)\|;
				\label{eq:nominal_tracking_bound}
			\end{align}
			\item if $\chi^{\star}<1$, the same ellipsoid is robustly positively invariant for every admissible $\mathbf d_e$, and
			\begin{align}
				\sqrt{V(t)}\leq& e^{-a^{\star}t/2}\sqrt{V(0)}
				+\chi^{\star}\sqrt{c^{\star}}
				\left(1-e^{-a^{\star}t/2}\right);
				\label{eq:robust_tracking_bound}
			\end{align}
			consequently,
			$\limsup_{t\to\infty}V(t)\leq(\chi^{\star})^{2}c^{\star}$;
			\item throughout the certified ellipsoid,
			$\z_{\mathrm{t}}+\z_{\mathrm{s}}+\z_{\mathrm{a}}$
			remains in $\Z(\boldsymbol{\alpha})$.
		\end{enumerate}
	\end{theorem}
	\begin{proof}
		By \eqref{eq:structured_synthesis}, the tracking ellipsoid $\mathcal{E}_{\boldsymbol{\theta}^{\star}}$ lies inside $\mathcal{E}_{\mathrm{loc}}$. From \eqref{eq:structured_feedback} and \eqref{eq:structured_level}, the maximal feedback norm on the ellipsoid is $\sup_{V\leq c^{\star}}\|\z_{\mathrm{s}}\|=\sqrt{c^{\star}}\|\mathbf{K}^{\star}(\mathbf{P}^{\star})^{-1/2}\|_{2}=\epsilon_{c}$.
		By \eqref{eq:control_uncertainty_bounds}, the combined demand satisfies $\|\z_{\mathrm{s}}+\z_{\mathrm{a}}\|\leq\epsilon_c+\bar z_{\mathrm{a}}=\epsilon_{\T_{z}}^{\star}$. Since $\z_{\mathrm{t}}\in\T_z$, \Cref{thm:uniform} guarantees actuator feasibility throughout the ellipsoid.

		Along the dynamics \eqref{eq:normalized_local_error} with feedback \eqref{eq:structured_feedback}, the Lyapunov identity \eqref{eq:structured_lyapunov} yields $\dot V=-\bar{\mathbf{x}}_{e}^{\top}\mathbf{Q}_{L}\bar{\mathbf{x}}_{e}+2\bar{\mathbf{x}}_{e}^{\top}\mathbf{P}^{\star}\bar{\mathbf{B}}_{0}\mathbf{d}_{e}$. For $\mathbf{d}_{e}\equiv\zero_{6}$, $\dot V\leq-a^{\star}V$ proves positive invariance. The quadratic bounds $\lambda_{\min}(\mathbf{P}^{\star})\|\bar{\mathbf{x}}_{e}\|^{2}\leq V\leq\lambda_{\max}(\mathbf{P}^{\star})\|\bar{\mathbf{x}}_{e}\|^{2}$ then establish \eqref{eq:nominal_tracking_bound}.

		Applying the Cauchy--Schwarz inequality with \eqref{eq:control_uncertainty_bounds} bounds the Lyapunov derivative:
		\begin{align}
			\dot V
			\leq&-a^{\star}V+2g^{\star}\bar d_{e}\sqrt{V}.
			\label{eq:tracking_dissipation}
		\end{align}
		Scalar comparison for $s=\sqrt{V}$ in \eqref{eq:tracking_dissipation}, using $\chi^{\star}$ from \Cref{thm:margin_tracking}, establishes \eqref{eq:robust_tracking_bound}. At $V=c^{\star}$, \eqref{eq:tracking_dissipation} reduces to $\dot V\leq-a^{\star}c^{\star}(1-\chi^{\star})<0$ when $\chi^{\star}<1$, proving robust positive invariance. The stated ultimate bound follows by taking $t\to\infty$ in \eqref{eq:robust_tracking_bound}.
	\end{proof}
	\begin{remark}
		The margin enters the synthesis through the certified level $c_{\boldsymbol{\theta}}$ and the local-chart constraint. A morphology with larger residual authority permits a larger tracking ellipsoid for the same gain family and state scaling. Nominal exponential convergence is recovered when the residual dynamic disturbance vanishes; a persistent disturbance produces the explicit ultimate bound in \eqref{eq:robust_tracking_bound}, rather than convergence to the origin. The distinction between $\z_{\mathrm{a}}$ and $\mathbf{d}_{e}$ prevents an additional commanded-wrench demand from being conflated with the residual matched disturbance acting on the tracking dynamics. \Cref{thm:margin_tracking} certifies existence of an admissible actuator vector; online realization may use the slack-maximizing allocator in \eqref{eq:slack_lp} or any map returning $\boldsymbol{\lambda}\in\Lambda$ with $\A\boldsymbol{\lambda}=\z_{\mathrm{t}}+\z_{\mathrm{s}}+\z_{\mathrm{a}}$.
	\end{remark}
	The corollary below concerns commanded wrench demand; an exogenous disturbance entering the state dynamics is treated separately in \Cref{thm:margin_tracking}.
	\begin{corollary}
		\label{cor:robust_filter}
		Assume \eqref{eq:h_rep}. Let $\bar{z}_{\mathrm{s}}$ bound the feedback wrench and let $\bar{z}_{\mathrm{a}}$ bound any additional \emph{normalized commanded} wrench used for compensation, regularization, or task shaping. If $\varphi_{\T_z}^{\star}(\boldsymbol{\alpha})\geq\bar{z}_{\mathrm{s}}+\bar{z}_{\mathrm{a}}$, then $\z_{\mathrm{t}}+\z_{\mathrm{s}}+\z_{\mathrm{a}}$ is actuator feasible for every admissible combination. Provided the tightened constraint set is nonempty, a desired command can be projected through
		\begin{equation}
			\z_{\mathrm{c}}^{\star}\in
			\argmin_{\substack{\z\in\mathcal{C}\\
				\mathbf{a}_{j}^{\top}\z\leq b_{j}-\epsilon_{\mathrm{r}}\|\mathbf{a}_{j}\|,
				~j=1,\ldots,q}}
			\|\z-\z_{\mathrm{c}}\|^{2},
			\label{eq:margin_filter}
		\end{equation}
		where $\epsilon_{\mathrm{r}}=\bar{z}_{\mathrm{s}}+\bar{z}_{\mathrm{a}}$ and $\mathcal{C}$ is a nonempty closed convex set containing any additional task constraints. The stated nonemptiness assumption applies to the intersection of $\mathcal C$ with the tightened half spaces.
	\end{corollary}
	\begin{proof}
		Applying the triangle inequality to the bounds in \Cref{cor:robust_filter}, with \eqref{eq:signed_uniform_margin}, one obtains $\|\z_{\mathrm{s}}+\z_{\mathrm{a}}\|\leq\bar{z}_{\mathrm{s}}+\bar{z}_{\mathrm{a}}\leq\inf_{\z\in\T_z}\varphi_{\Z}(\boldsymbol{\alpha},\z)$. By \Cref{thm:uniform}, every admissible combined command lies in $\Z(\boldsymbol{\alpha})$. For the representation \eqref{eq:h_rep}, the constraints in \eqref{eq:margin_filter} are equivalent, through \eqref{eq:signed_margin_z}, to $\varphi_{\Z}(\boldsymbol{\alpha},\z)\geq\epsilon_{\mathrm{r}}$. Thus, the projection preserves the required reserve.
	\end{proof}
	The normalized allocation map and feasible set were defined in \eqref{eq:normalized_feasible_set_control}; the pointwise signed margin of the feasible set $\varphi_{\Z}(\boldsymbol\alpha,\z)$ is given by \eqref{eq:signed_margin_z}. For morphology optimization, configurations that violate \Cref{ass:rank} are assigned $\varphi_{\Z}=-\infty$; equivalently, the optimization can be restricted to the rank-six subset of $\Aset$.
	For the numerical architecture, the modules lie on a ring of radius $\ell$ at azimuths $\psi_{i}=2\pi\left(i-1\right)/N$. Define radial and tangential directions $\mathbf{e}_{r,i}$ and $\mathbf{e}_{t,i}$ in the rotor plane. A mixed tilted axis is parameterized by $\mathbf{b}_{i}(\gamma_{i},\eta_{i})=\cos{\gamma_{i}}\e_{3}+\sin{\gamma_{i}}(\cos{\eta_{i}}\mathbf{e}_{r,i}+\sin{\eta_{i}}\mathbf{e}_{t,i})$.
	Setting $\gamma_{i}=0$ recovers a coplanar vehicle. Constant angles define a fixed-tilt design, whereas admissible time-varying angles represent a reconfigurable morphology.
	\begin{lemma}
		\label{prop:coplanar}
		When parallel thrust directions are held fixed, $\rank(\M)\leq4$, so the attainable wrench set admits no six-dimensional residual neighborhood.
	\end{lemma}
	\begin{proof}
		In \eqref{eq:wrench_sum_short}, parallel axes restrict the force component to a one-dimensional subspace. The moment component has dimension at most three, so the allocation range has dimension at most four. The nonsingular normalization in \eqref{eq:normalized_feasible_set_control} preserves the allocation rank. Translation in \eqref{eq:residual_set} preserves affine dimension, so $\Z_{\mathrm{r}}$ has empty interior in $\R^{6}$.
	\end{proof}
	For a horizontal unit direction $\mathbf{n}$, define the normalized pushing family
	\begin{align}
		\T_{\mathrm{push}}
		=&\left\{\mathbf{D}_{w}
		\begin{bmatrix}mg\e_{3}+F\mathbf{n}\\\boldsymbol{\tau}_{c}\end{bmatrix}:
		\begin{gathered}
			F\in[0,F_{\max}],\\
			\|\boldsymbol{\tau}_{c}\|\leq\tau_{\max}
		\end{gathered}\right\}.
		\label{eq:t_push}
	\end{align}
	Specializing \eqref{eq:signed_uniform_margin} to \eqref{eq:t_push} defines the worst signed push margin $\varphi_{\mathrm{push}}^{\star}(\boldsymbol{\alpha}):=\inf_{\z\in\T_{\mathrm{push}}}\varphi_{\Z}(\boldsymbol{\alpha},\z)$.
	For $\varphi_{\mathrm{push}}^{\star}\geq0$, the reserve equals $\epsilon_{\mathrm{push}}^{\star}=\varphi_{\mathrm{push}}^{\star}$; a negative value certifies infeasibility, and clipping the margin to zero provides only a diagnostic. We write $\varphi_{\mathrm{push}}^{\star}(\boldsymbol\alpha;F_{\max},\tau_{\max})$ when displaying the task bounds. A fixed morphology may be selected through $\boldsymbol{\alpha}^{\star}\in\argmax_{\boldsymbol{\alpha}\in\Aset}\varphi_{\mathrm{push}}^{\star}(\boldsymbol{\alpha})$. The largest push compatible with a required reserve $\epsilon_{0}>0$ is
	\begin{align}
		F_{\mathrm{push}}^{\star}
		=&\sup\left\{F_{\max}:\varphi_{\mathrm{push}}^{\star}
		\left(\boldsymbol{\alpha};F_{\max},\tau_{\max}\right)\geq\epsilon_{0}\right\}.
		\label{eq:max_push}
	\end{align}
	Unlike a maximum-force metric, \eqref{eq:max_push} withholds authority in all normalized wrench directions. If facets are unavailable, let $\z_0\in\intset(\Z(\boldsymbol{\alpha}))$ and $\|\mathbf{d}\|=1$. Define the directional reachability $s^{\star}(\z_0,\mathbf{d}):=\max_{s,\boldsymbol{\lambda}}s$ subject to $\A\boldsymbol{\lambda}=\z_{0}+s\mathbf{d}$, $\boldsymbol{\lambda}\in\Lambda$, and $s\geq0$. The exact Euclidean radius is the infimum over the unit sphere. A finite direction set provides a diagnostic unless a covering-error bound is included; the computations below instead use exact half-space models for both architectures.
	
	Finally, define the best fixed and pointwise adaptive quasi-static margins over $\T_z$ as
	\begin{align}
		\varphi_{\mathrm{fix}}^{\star}
		=&\sup_{\boldsymbol{\alpha}\in\Aset}
		\inf_{\z\in\T_{z}}\varphi_{\Z}\left(\boldsymbol{\alpha},\z\right),
		~~
		\varphi_{\mathrm{ad}}^{\star}
		=\inf_{\z\in\T_{z}}
		\sup_{\boldsymbol{\alpha}\in\Aset}\varphi_{\Z}\left(\boldsymbol{\alpha},\z\right).
		\label{eq:fixed_margin}
	\end{align}
	\begin{lemma}
		\label{prop:nonreduction}
		For compact $\T_{z}$ and $\Aset$, $\varphi_{\mathrm{ad}}^{\star}\geq\varphi_{\mathrm{fix}}^{\star}$.
	\end{lemma}
	\begin{proof}
		For every fixed $\boldsymbol{\alpha}\in\Aset$, $\varphi_{\Z}(\boldsymbol{\alpha},\z)\leq\sup_{\widetilde{\boldsymbol{\alpha}}\in\Aset}\varphi_{\Z}(\widetilde{\boldsymbol{\alpha}},\z)$ holds for all $\z\in\T_z$. Taking the infimum over $\z$ and then the supremum over $\boldsymbol{\alpha}$ establishes the claimed order from the definitions in \eqref{eq:fixed_margin}.
	\end{proof}
	\Cref{prop:nonreduction} compares quasi-static wrench authority. Servo-rate limits, transient singularities, and closed-loop stability must be addressed in a full dynamic synthesis. 
	
	\section{Numerical Evaluation}
	\label{sec:numerics}
	Offline reserve and controller synthesis are followed by closed-loop wall-contact tests of the predicted actuator limits. Throughout, $\T_z=\T_{\mathrm{push}}$, denoted by $\T$ in the figures. The offline study uses $m=3.3~\mathrm{kg}$, $\ell=0.35~\mathrm{m}$, $c_i=c_{\tau}=0.025~\mathrm{m}$, $F_{0}=mg$, and $\ell_{0}=\ell$. Tasks satisfy $F\in[0,0.15mg]$ along $\mathbf{n}=\e_{1}$ with $\|\boldsymbol{\tau}_{c}\|\leq0.04\ell mg$; hence, $\rho_{\max}=F_{\max}/(mg)=0.15$ and $\mu_{\max}=\tau_{\max}/(\ell_0mg)=0.04$. A normalized reserve $\epsilon$ corresponds to physical reserve limits of $\epsilon F_{0}$ in force and $\epsilon \ell_{0} F_{0}$ in moment. With $F_{0} = 32.4~\mathrm{N}$, $\ell_{0}F_{0} = 11.3~\mathrm{N.m}$, and $\epsilon_{0}=0.025$, the required reserve is $0.81~\mathrm{N}$ in force and $0.28~\mathrm{N.m}$ in moment. Equally spaced rotors satisfy $\mathbf{r}_i=\ell\mathbf{e}_{r,i}$, $\eta_i=(-1)^{i+1}\pi/4$, and $\sigma_i=(-1)^{i+1}$. The hexarotor $\mathsf{H}_6$ and octarotor $\mathsf{O}_8$ allow $13~\mathrm{N}$ per rotor, totaling $78~\mathrm{N}$ and $104~\mathrm{N}$, respectively. The capacity-matched $\mathsf{O}_{8,\mathrm{eq}}$ retains the octarotor geometry and two-dimensional allocation null space with $\lambda_{\max,8}^{\mathrm{eq}}=\frac{6}{8}(13)=9.75$~N per rotor. Thus, $(\mathsf{H}_6,\mathsf{O}_8)$ compares complete architectures at equal per-rotor limits; $(\mathsf{H}_6,\mathsf{O}_{8,\mathrm{eq}})$ assesses geometry and redundancy at equal installed capacity.

	For each exact half-space model $\mathbf{H}\z\leq\mathbf{b}$, write row $j$ as $\mathbf{a}_j^{\top}$, with $\mathbf{a}_j=[\mathbf{a}_{j,f}^{\top}~\mathbf{a}_{j,\tau}^{\top}]^{\top}$. With $\z_h=\mathbf{D}_w\begin{bmatrix}mg\e_3^{\top}&\zero_3^{\top}\end{bmatrix}^{\top}$, the task support is $\mathbf{a}_j^{\top}\z_h+\max\{0,\rho_{\max}\mathbf{a}_{j,f}^{\top}\mathbf{n}\}+\mu_{\max}\|\mathbf{a}_{j,\tau}\|$. Applying \Cref{thm:uniform} covers the continuous force interval and torque ball analytically. For $N=6$, half-spaces follow from $\A^{-1}$ and the thrust box; for $N=8$, facets are enumerated from independent five-generator subsets.
	\begin{table}[h!]
		\centering
		\caption{Architecture and equal-total-thrust comparison. Values at $\gamma=36^{\circ}$ use the common task set; parenthetical angles identify the optimizer of each fine-grid sweep.}
		\label{tab:equal_total_comparison}
		\resizebox{\linewidth}{!}{%
			\setlength{\tabcolsep}{3.2pt}
			\begin{tabular}{lccccc}
				\toprule
				Case & $\sum_i\lambda_{i,\max}$ & $\varphi_{\mathrm{push}}^{\star}(36^{\circ})$ & $F_{\mathrm{push}}^{\star}(36^{\circ})/mg$ & $\max_{\gamma}\varphi_{\mathrm{push}}^{\star}$ & $\max_{\gamma}F_{\mathrm{push}}^{\star}/mg$\\
				\midrule
				$\mathsf{H}_6$ & $78$ & $0.050$ & $0.188$ & $0.050~(36.2^{\circ})$ & $0.188~(37.1^{\circ})$\\
				$\mathsf{O}_8$ & $104$ & $0.112$ & $0.260$ & $0.130~(46.4^{\circ})$ & $0.322~(55.1^{\circ})$\\
				$\mathsf{O}_{8,\mathrm{eq}}$ & $78$ & $0.048$ & $0.179$ & $0.065~(48.2^{\circ})$ & $0.213~(49.1^{\circ})$\\
				\bottomrule
			\end{tabular}%
		}
	\end{table}
	
	Both coplanar architectures have rank four, and increasing tilt can restore rank six before task feasibility. At $\gamma=36^{\circ}$, $\mathsf{O}_{8,\mathrm{eq}}$ retains $96.9\%$ and $95.4\%$ of the hexarotor's task margin and certified push, respectively (\Cref{tab:equal_total_comparison} and \Cref{fig:equal_total_morphology}). For the prescribed task, optimizing tilt reverses the comparison: within the stated morphology family and actuator limits, the capacity-matched octarotor attains a $30.3\%$ larger maximum task margin and a $13.3\%$ larger maximum certified push than the optimized hexarotor. These margins assess the task and reserve relative to hover within the finite thrust box, beyond structural span and conditioning.
	\begin{figure}[tb]
		\centering
		\begin{subfigure}[t]{0.48\linewidth}
			\centering
			\includegraphics[width=\linewidth]{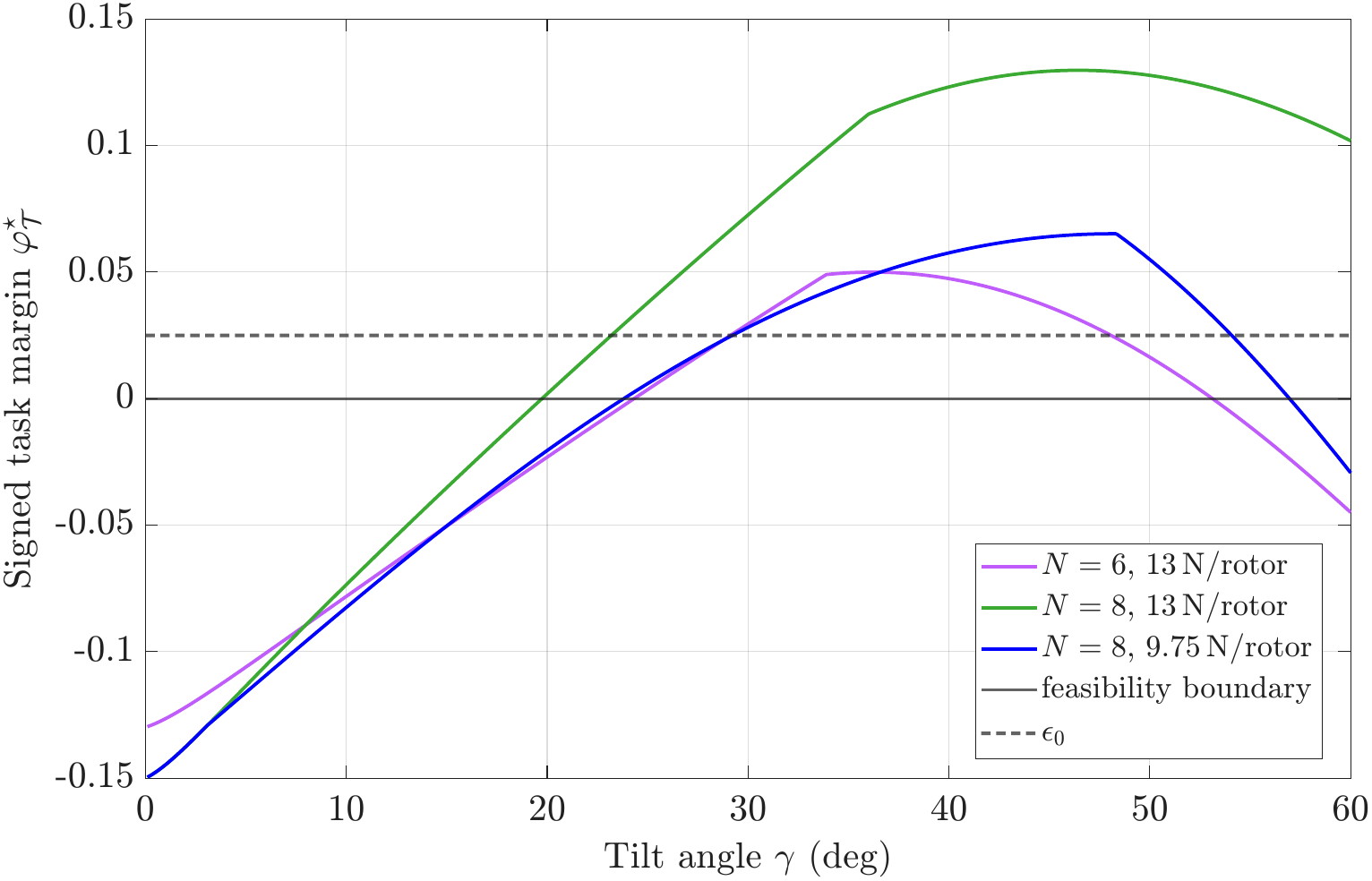}
			\caption{Signed task margin.}
			\label{fig:equal_total_margin}
		\end{subfigure}\hfill
		\begin{subfigure}[t]{0.48\linewidth}
			\centering
			\includegraphics[width=\linewidth]{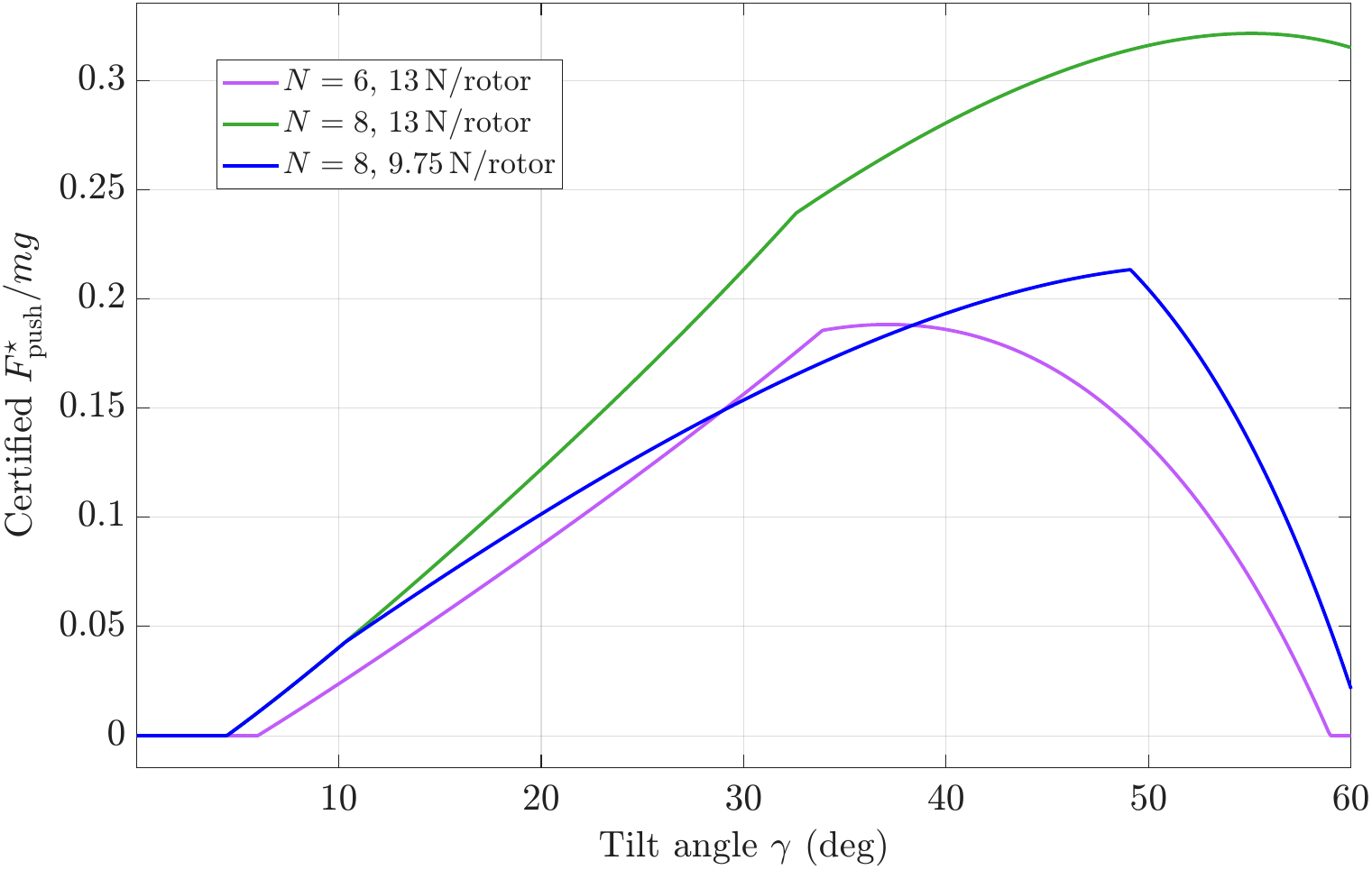}
			\caption{Largest push preserving $\epsilon_0$.}
			\label{fig:equal_total_push}
		\end{subfigure}
		\caption{Task margin and certified push versus tilt for $\mathsf{H}_6$, $\mathsf{O}_8$, and the capacity-matched $\mathsf{O}_{8,\mathrm{eq}}$.}
		\label{fig:equal_total_morphology}\vspace{-4mm}
	\end{figure}
	
	Scaling thrust bounds changes $\Z$ while preserving $\A$, so $\mathsf{O}_8$ and $\mathsf{O}_{8,\mathrm{eq}}$ share the singular-value curve in \Cref{fig:conditioning_certificates}. Signed margins and interiority bounds quantify the authority lost through capacity normalization.

	At $\gamma=36^{\circ}$ and $\z_0=[0.15~0~1~0~0~0]^{\top}$, the exact point margins for $\mathsf{H}_6$, $\mathsf{O}_8$, and $\mathsf{O}_{8,\mathrm{eq}}$ are $0.076$, $0.137$, and $0.073$, respectively. Slack maximization raises the unscaled octarotor certificate from $0.072$ (pseudoinverse allocation) to $0.105$, a $45.3\%$ increase; the capacity-matched bounds are $0.063$ and $0.068$. Thus, null-space redistribution improves the allocation-interiority certificate without enlarging the polytope, whose exact margin measures authority independently of allocation.
	\begin{figure}[tb]
		\centering
		\begin{subfigure}[t]{0.48\linewidth}
			\centering
			\includegraphics[width=\linewidth]{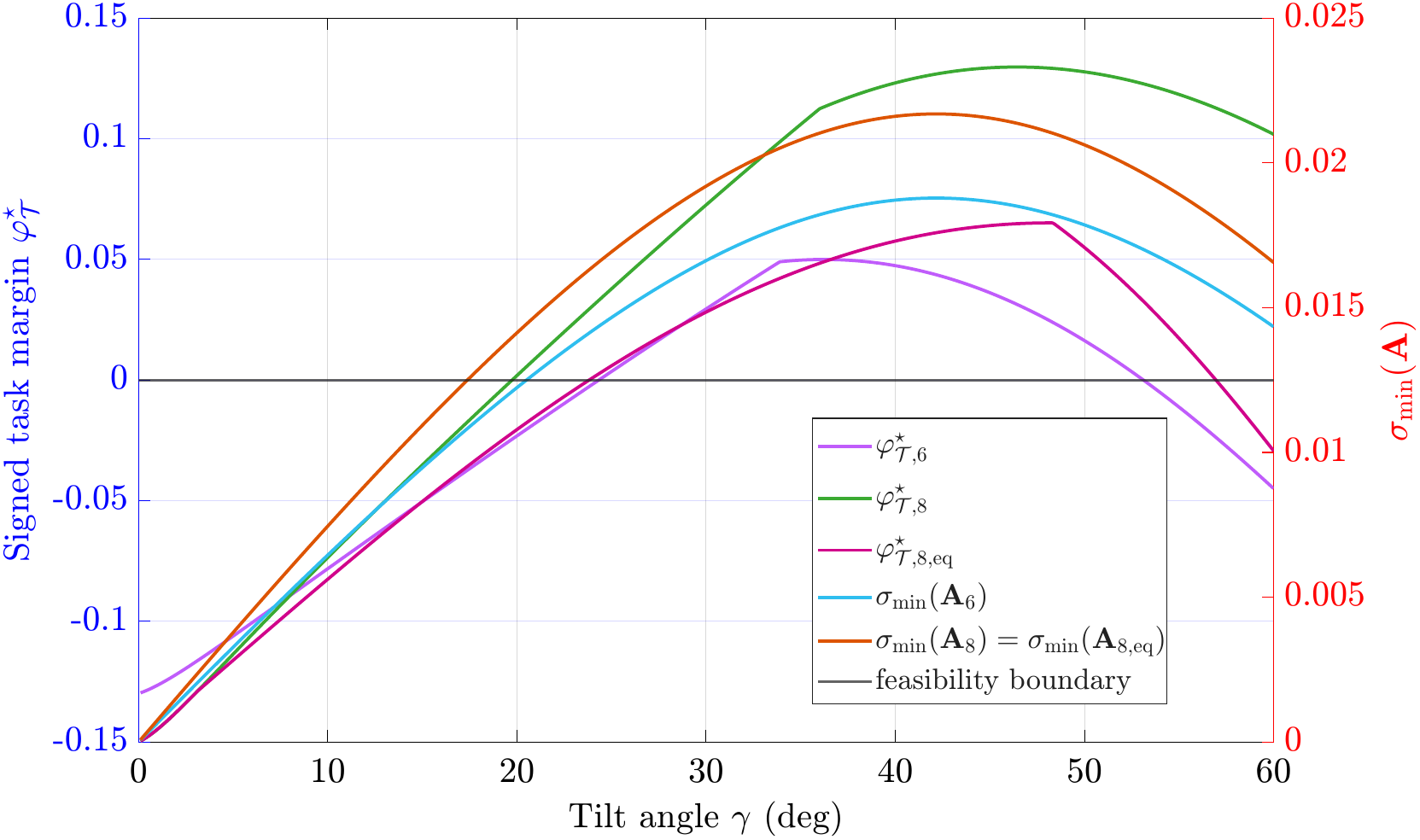}
			\caption{Conditioning and task margin.}
			\label{fig:conditioning_n6n8}
		\end{subfigure}\hfill
		\begin{subfigure}[t]{0.48\linewidth}
			\centering
			\includegraphics[width=\linewidth]{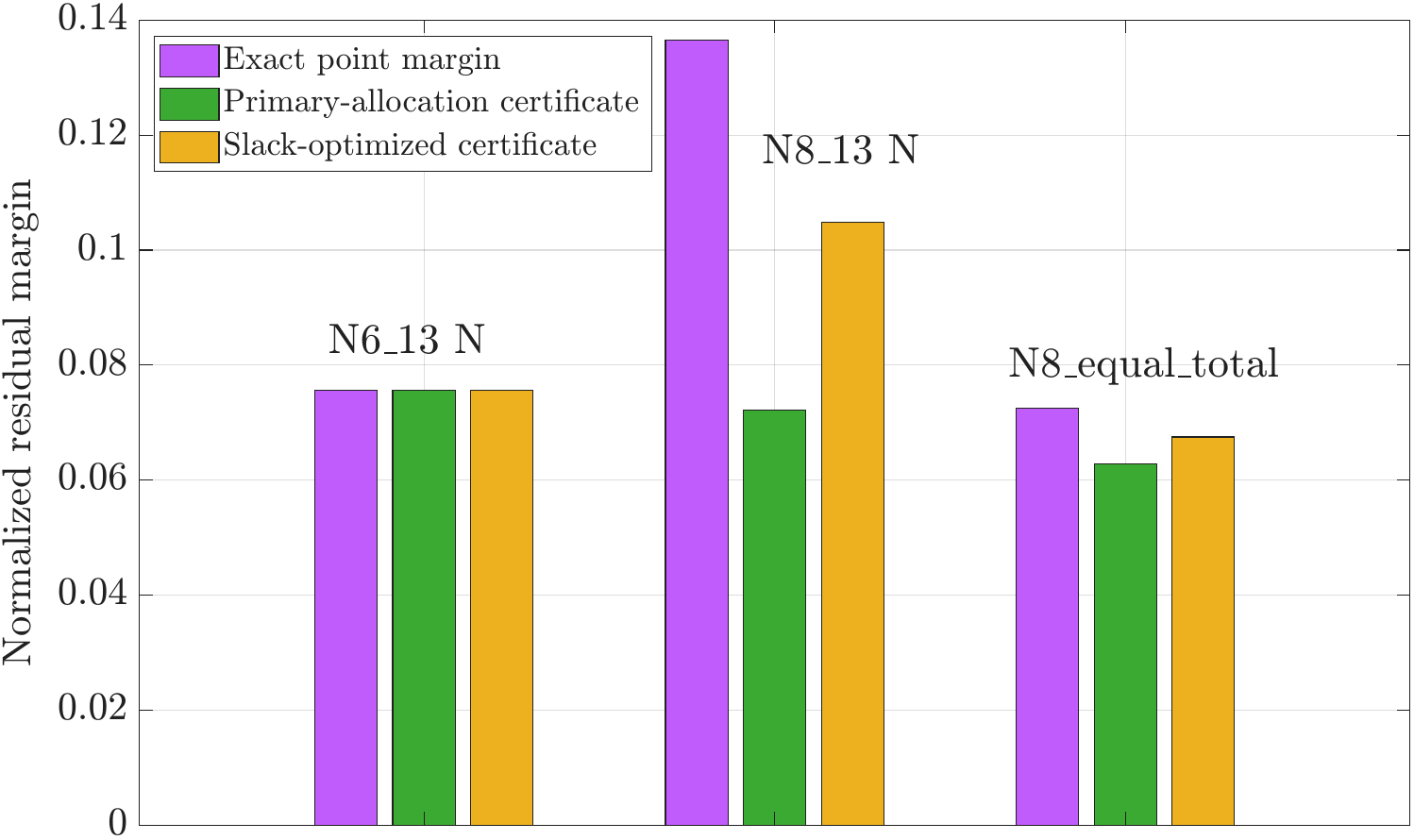}
			\caption{Point margins and certificates at $\gamma=36^{\circ}$.}
			\label{fig:equal_total_certificates}
		\end{subfigure}
		\caption{Residual authority under different thrust bounds and allocations. In (a), both octarotors share the singular-value curve.}
		\label{fig:conditioning_certificates}\vspace{-3mm}
	\end{figure}
	
	The force--torque maps in \Cref{fig:robust_maps} compare geometry and capacity under equal per-rotor limits. A deterministic $141\times121$ sweep covers $\rho_{\max}\in[0,0.38]$ and $\mu_{\max}\in[0,0.12]$, totaling $17{,}061$ exact signed-margin evaluations per architecture and $34{,}122$ overall, with analytic torque-ball support at each point. Common contours and color limits support comparison: solid lines mark feasibility, $\varphi_{\mathrm{push}}^{\star}=0$, and dash-dot lines mark the required reserve, $\varphi_{\mathrm{push}}^{\star}=\epsilon_0=0.025$. The positive margin preserves corrective authority for auxiliary commands and bounded model mismatch. The same facet calculations apply when morphology, actuator sizing, or reserve requirements change.
	\begin{figure}[tb]
		\centering
		\begin{subfigure}[t]{0.48\linewidth}
			\centering
			\includegraphics[width=\linewidth]{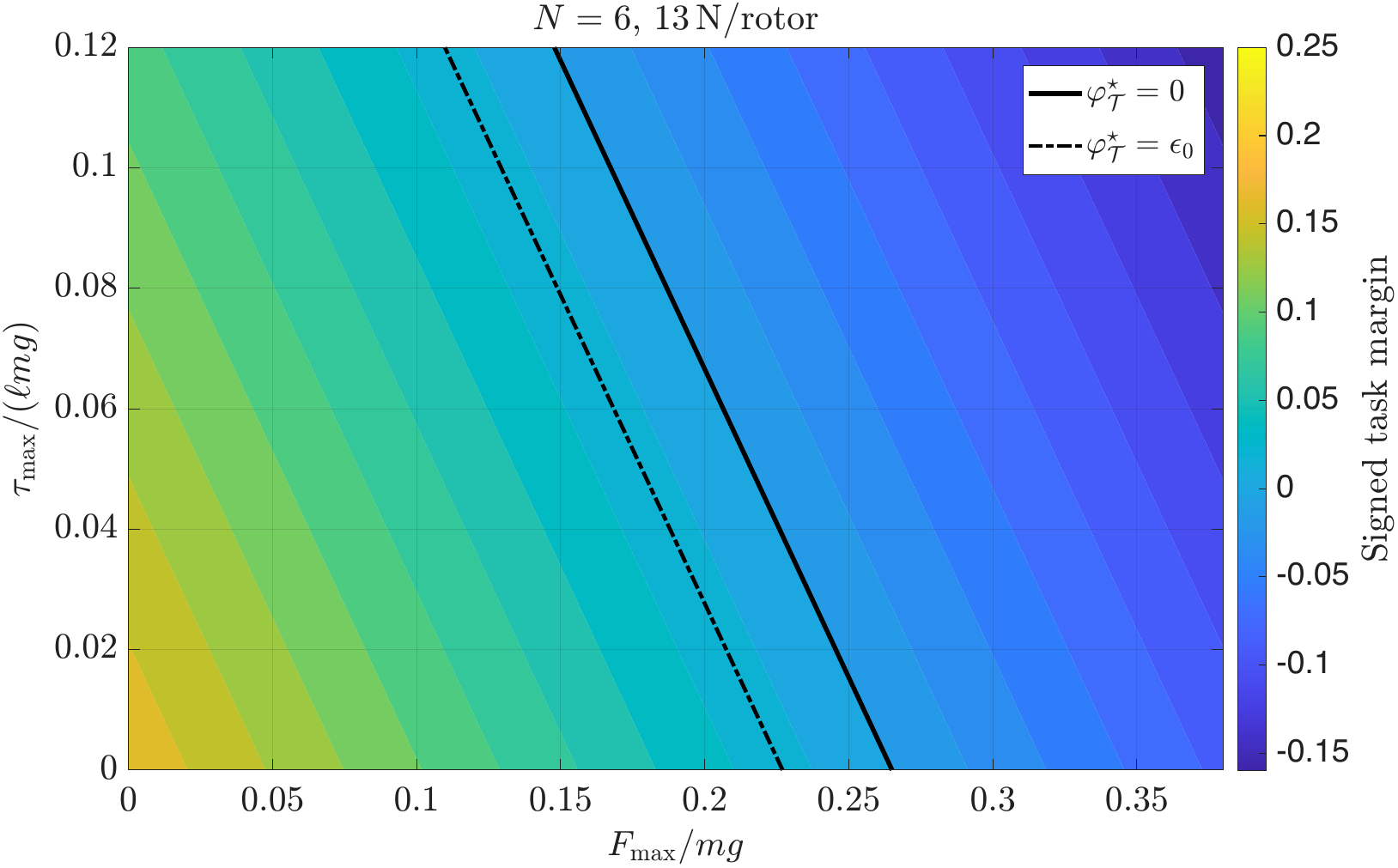}
			\caption{$\mathsf{H}_6$: $13~\mathrm{N}$ per rotor.}
			\label{fig:robust_n6}
		\end{subfigure}\hfill
		\begin{subfigure}[t]{0.48\linewidth}
			\centering
			\includegraphics[width=\linewidth]{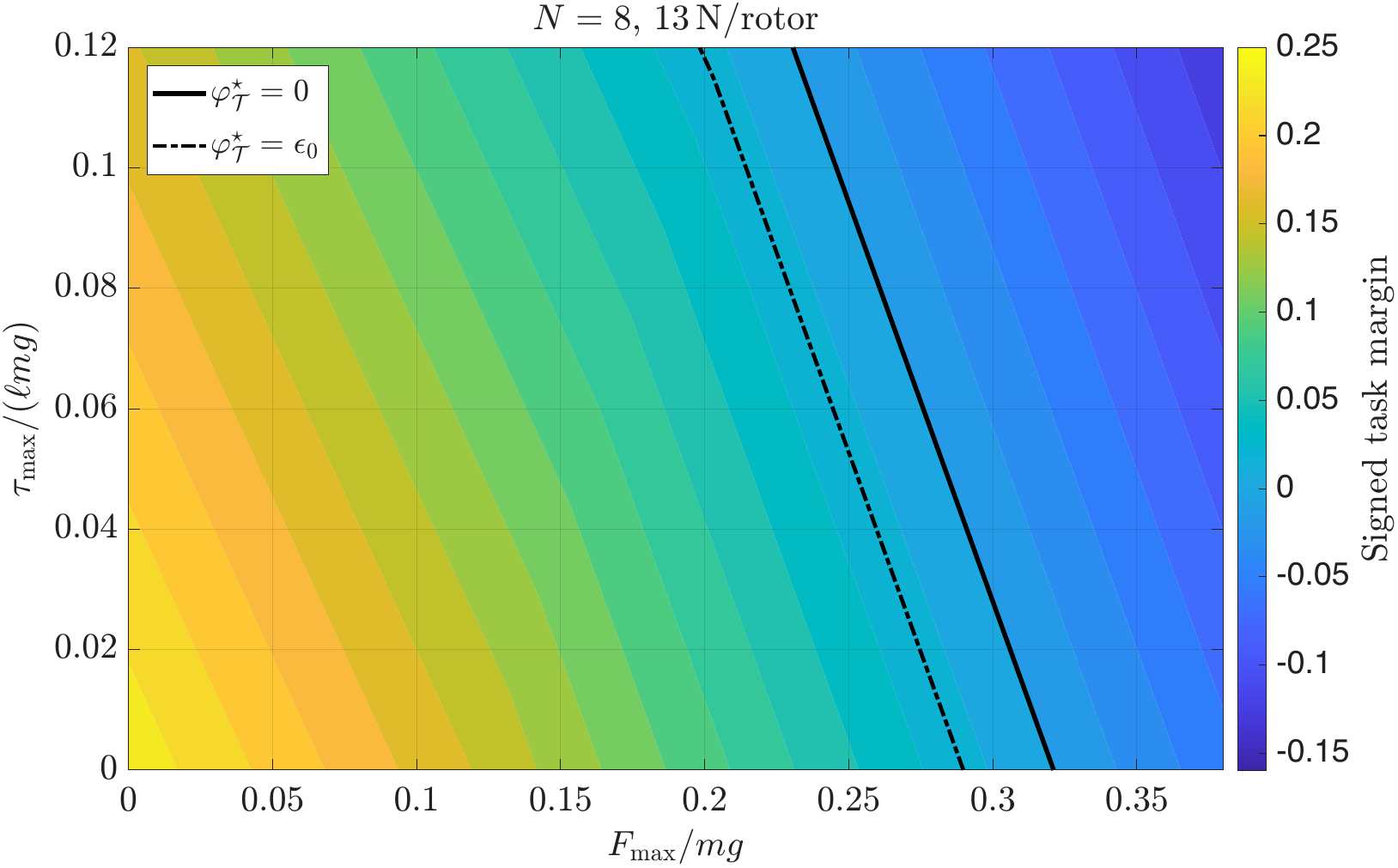}
			\caption{$\mathsf{O}_8$: $13~\mathrm{N}$ per rotor.}
			\label{fig:robust_n8}
		\end{subfigure}
		\caption{Task-margin maps at $\gamma=36^{\circ}$ under equal per-rotor thrust limits.}
		\label{fig:robust_maps}\vspace{-3mm}
	\end{figure}
	
	The controller study uses the canonical normalized double integrator $\bar{\mathbf{A}}_0=\begin{bmatrix}\zero&\mathbf{I}_6\\\zero&\zero\end{bmatrix}$ and $\bar{\mathbf{B}}_0=\begin{bmatrix}\zero\\\mathbf{I}_6\end{bmatrix}$, with $\mathbf{Q}_L=\mathbf{I}_{12}$, $\mathbf{X}_{\mathrm{loc}}=\mathbf{I}_{12}$, and $\bar{z}_{\mathrm{a}}=\bar{d}_e=0$. For $\mathbf{K}(k_p,k_d)=\begin{bmatrix}k_p\mathbf{I}_6&k_d\mathbf{I}_6\end{bmatrix}$ and $\beta=0.2$, the structured search returns $k_p=0.076$ and $k_d=0.4$ for all cases. At $\gamma=36^{\circ}$, the certified levels $c_{6}^{\star}=0.200$, $c_{8}^{\star}=1.012$, and $c_{8,\mathrm{eq}}^{\star}=0.188$ reflect an approximately fivefold unscaled octarotor increase from geometry and installed thrust; the capacity-matched level is $93.8\%$ of the hexarotor value. Morphology optimization can reverse the common-tilt comparison (\Cref{fig:equal_total_morphology}) when the local-chart constraint is inactive.

	The synthesis in \eqref{eq:structured_synthesis} connects reserve-based gains to the local tracking guarantees of \Cref{thm:margin_tracking}. A useful actuator-feasible error set requires joint assessment of morphology and allocation with controller synthesis.

    Synchronized ROS~2 and Gazebo wall-contact simulations take the octarotor and variable-tilt quadrotor from a common operating point beyond their predicted feasible boundaries. We use normalized push $\rho=F/(mg)$ and certified limit $\rho^\star:=F_{\mathrm{push}}^\star/(mg)$. The quantity $F_{x,\max}$ denotes the maximum zero-moment push along $\e_1$ compatible with hover.

    The simulated $\mathsf{O}_8$ retains the offline geometry ($m=3.3$~kg, $\gamma=36^\circ$) with a $12.5$~N per-rotor limit. The variable-tilt quadrotor $\mathsf{VTQ}_4$ has $m=3.5$~kg, $L_x=0.17$~m, $L_y=0.15$~m, hub height $h=0.05$~m, reaction-torque coefficient $c_i=c_\tau=0.020$~m, and $|\alpha|,|\beta|\le36^\circ$. Both have installed thrust $\sum_i\lambda_{i,\max}=100$~N ($8\times12.5$~N and $4\times25$~N). The normalization uses $\ell_{0,\mathsf{O}_8}=0.35$~m and $\ell_{0,\mathsf{VTQ}_4}=\sqrt{L_x^2+L_y^2}\approx0.227$~m ($F_{0} = 34.3~\mathrm{N}$ and $\ell_{0} F_{0} = 7.78~\mathrm{N.m}$), with common uncertainty $\mu_{\max}=0.04$ and reserve $\epsilon_0=0.025$.

    With the servos held, the quadrotor's parallel thrust axes yield
    $\operatorname{rank}(\mathbf{A})=4$ (Lemma~3); six-dimensional local
    authority therefore requires both servo coordinates. The $\mathsf{O}_8$
    reserve follows from Theorem~3 on the full normalized zonotope. For
    $\mathsf{VTQ}_4$, define
    $\mathbf{u}=[\boldsymbol{\lambda}^{\top}\ \alpha\ \beta]^{\top}$ and
    $\mathbf{J}_u:=\partial \mathbf{z}/\partial \mathbf{u}$; the resulting
    reserve certificate is local to the operating morphology. Since singular
    values depend on actuator-coordinate scaling, the conditioning comparison
    in Table~II is performed in dimensionless actuator coordinates. For
    $\mathsf{O}_8$, the thrust coordinates are scaled by
    $\lambda_{\max}$, so Table~II reports
    $\sigma_{\min}(\mathbf{A}\lambda_{\max})$. For $\mathsf{VTQ}_4$, define
    $\mathbf{D}_u := \operatorname{blkdiag} \left( \lambda_{\max}\mathbf{I}_4,\,s_\alpha,\,s_\beta\right)$
    with $s_\alpha=s_\beta=36^\circ$ expressed in radians; Table~II then
    reports $\sigma_{\min}(\mathbf{J}_u\mathbf{D}_u)$, evaluated at the
    nominal hover operating morphology.

    At baseline, the offline and closed-loop slack-maximized clearances both round to $2.866$~N. For $\mathsf{O}_8$, the facet-predicted roll--pitch diagonal minimizing $\varphi_{\mathrm{push}}^{\star}$ is worst among five tested torque directions, and the simulated margin agrees with the offline value within $10^{-6}$. Fixed quadrotor servos permit only a one-dimensional force span; releasing the tilts realizes the lateral demand at the predicted morphology. The octarotor realizes the same task with a margin of approximately $0.128$.

    \begin{table}[t]
        \centering
        \caption{Equal-thrust comparison ($100$~N installed) with $\mu_{\max}=0.04$ and $\epsilon_0=0.025$.}
        \label{tab:dyn}
        \setlength{\tabcolsep}{5pt}
        \begin{tabular}{lcc}
            \toprule
            Metric & $\mathsf{O}_8$ fixed & $\mathsf{VTQ}_4$ tilting\\
            \midrule
            Actuation rank                    & $6$      & $6$\\
            Actuator-scaled $\sigma_{\min}$  & $0.263$ & $0.129$\\
            $F_{x,\max}/(mg)$              & $0.309$ & $0.727$\\
            $\varphi_{\mathrm{push}}^{\star}$ at $\rho=0.15$ & $0.103$ & $0.044$\\
            $\rho^\star$ (certified)        & $0.247$ & $0.696$\\
            Limiting mechanism              & rotor thrust & $\alpha=+36^\circ$\\
            Beyond-boundary test            & $\rho=0.280$ & $\rho=0.730$\\
            \bottomrule
        \end{tabular}
    \end{table}
    
    The push sweep brackets the certified limits (\Cref{fig:d4} and \Cref{tab:dyn}). For $\mathsf{O}_8$, $\rho^\star\approx0.247$, with margins $\varphi_{\mathrm{push}}^{\star}=0.0252$ at $\rho=0.247$ and $0.023$ at $0.250$. A rotor-thrust facet limits the envelope; allocation clearance falls from $2.855$~N to $1.250$~N. For $\mathsf{VTQ}_4$, $\rho^\star\approx0.696$, with bracketing margins $0.026$ and $0.021$. The binding tilt cone $|\alpha|\le\alpha_{\max}$ yields
        \begin{equation}
         \sin\alpha_{\max}-\rho^\star\cos\alpha_{\max}=\epsilon_0 ,
        \label{eq:vtq4_cf}
    \end{equation}
    which reproduces $\rho^\star$. At $\rho=0.710$, servo and rotor clearances are $0.624^\circ$ and more than $8.3$~N, respectively, identifying tilt authority as the limiting constraint. With the certificate used only as a monitor, operation beyond $\rho^\star$ remains feasible until the realizability boundary: $\rho^\star$ marks exhaustion of the prescribed reserve.

    Both architectures have six-dimensional certificate-relevant authority. The reported normalized smallest singular value favors $\mathsf{O}_8$, while maximum directional force favors $\mathsf{VTQ}_4$; neither ranking identifies task-family reserve or the active constraint. At $\mu_{\max}=0.04$, the certified push ratio is $\rho^\star_{\mathsf{VTQ}_4}/\rho^\star_{\mathsf{O}_8}=2.81$. The force-only surface \eqref{eq:vtq4_cf} is independent of $\mu_{\max}$; the octarotor boundary decreases by approximately $22\%$ as $\mu_{\max}$ grows from $0$ to $0.08$, so the ratio varies from $2.51$ to $3.20$.

    \begin{figure}[tb]
        \centering
        \begin{subfigure}[t]{0.45\linewidth}
            \centering
            \includegraphics[width=\linewidth]{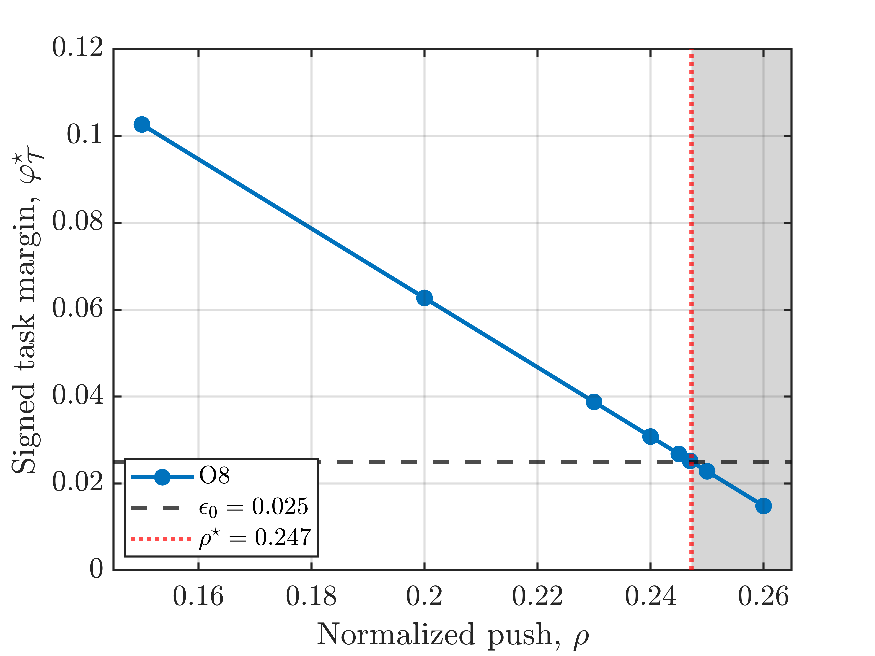}
             \caption{Reserve versus push for $\mathsf{O}_8$.}
        \end{subfigure}
        \begin{subfigure}[t]{0.45\linewidth}
            \centering
            \includegraphics[width=\linewidth]{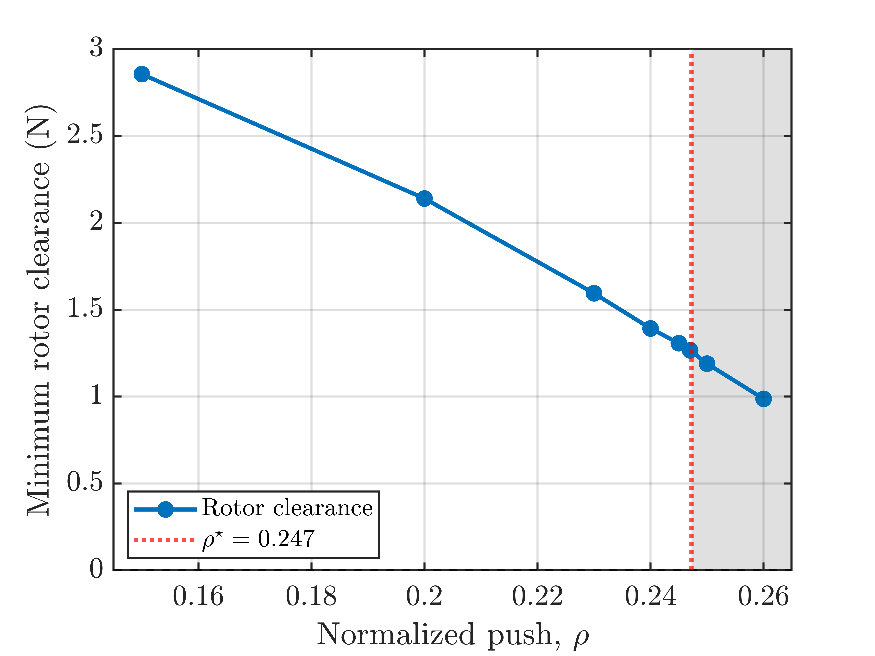}
            \caption{$\mathsf{O}_8$ rotor clearance.}    
        \end{subfigure}\par
        \begin{subfigure}[t]{0.23\textwidth}
            \centering
             \includegraphics[width=\linewidth]{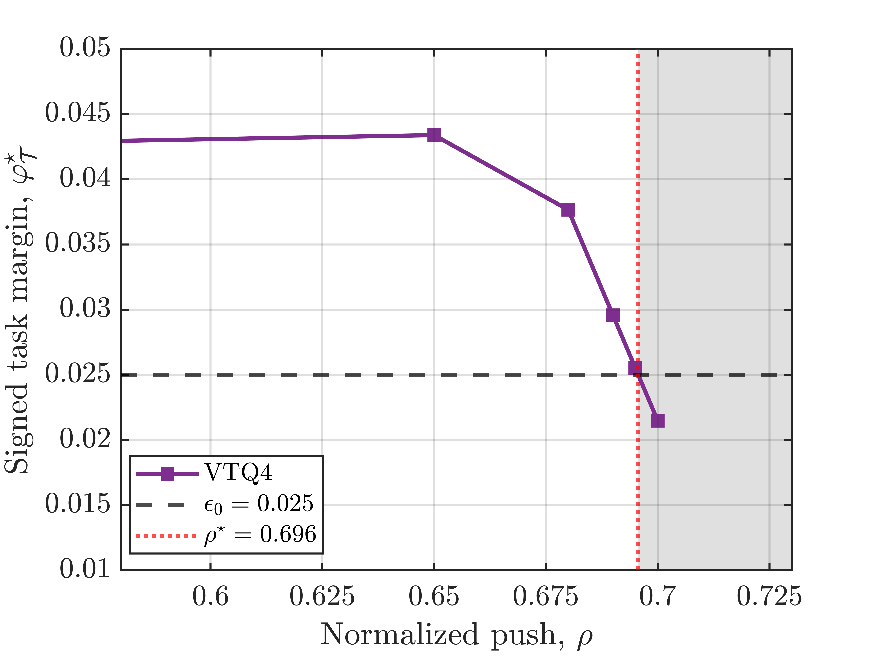}
             \caption{Reserve versus push for $\mathsf{VTQ}_4$.}
        \end{subfigure}
        \begin{subfigure}[t]{0.23\textwidth}
            \centering
            \includegraphics[width=\linewidth]{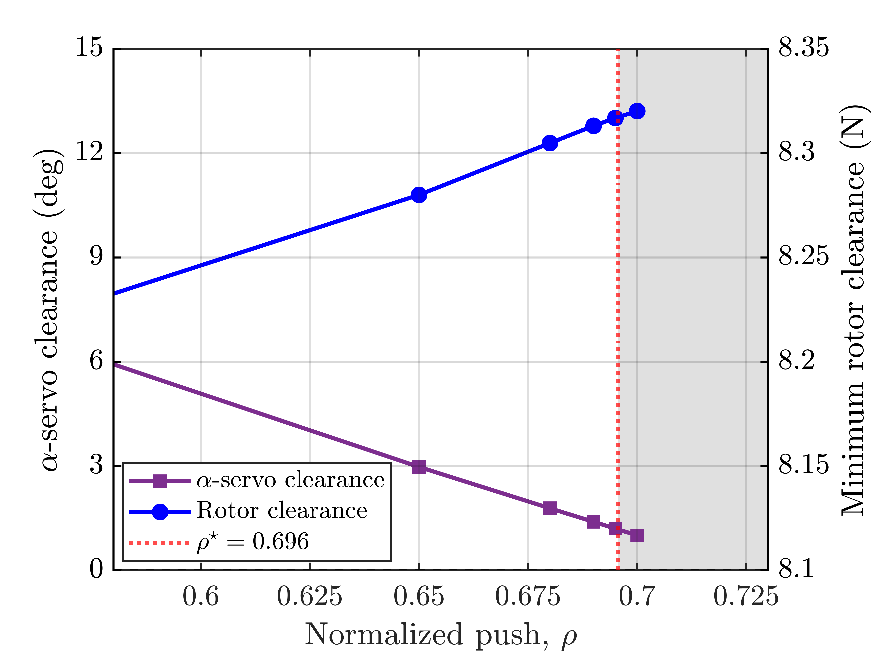} 
             \caption{$\mathsf{VTQ}_4$ servo and rotor clearance.}
        \end{subfigure}
        \caption{Certified push envelopes at equal installed thrust. Vertical lines mark offline thresholds; markers denote simulation samples, highlighted region denotes the post-certified region.}\vspace{-4mm}
        \label{fig:d4}
    \end{figure} 
    Beyond-boundary tests bypass margin-preserving projection, passing the complete task to the allocator under preregistered contradiction criteria. For $\mathsf{O}_8$ at $\rho=0.280$, the combined push-plus-binding-torque task has normalized torque $\boldsymbol{\mu}^{\star}:=\boldsymbol{\tau}_{c}^{\star}/(\ell_0mg)=\mu_{\max}[-1~1~0]^{\top}/\sqrt{2}$ and full-amplitude margin $\varphi_{\mathrm{push}}^{\star}=-1.140\times10^{-3}$ on the predicted facet. A contradiction requires unmodified task realization with normalized wrench residual below $10^{-6}$ and all eight rotors inside $[0,12.5]$~N. In \Cref{fig:d6}, clearance falls through $0.0144$~N at $t=54.211$~s and $2.14\times10^{-4}$~N at the last feasible sample. Lower and upper rotor bounds become active at the predicted facet; allocation first becomes infeasible at $t=54.261$~s, before the controller aborts at $t=54.751$~s.

    At $\rho=0.730>\tan{\alpha_{\max}}\approx0.727$, the quadrotor's nominal force requires $\alpha_{\mathrm{req}}\approx36.129^\circ$, exceeding the servo limit by $0.129^\circ$. Constraining thrust to $\mathbf{b}(36^\circ,0)$ predicts a minimum force-vector residual of $0.096$~N. A contradiction requires unmodified force realization within $10^{-3}$~N at nominal attitude, with $|\alpha|\le36^\circ$, all four rotors inside $[0,25]$~N, and no abort. In \Cref{fig:d6}, commanded and actual $\alpha$ settle at $36^\circ$; the measured residual of $0.096$~N agrees with the prediction within $1.03\,\mu$N. Rotor thrusts remain between $8.36$ and $12.90$~N, without abort.

    \begin{figure}[tb]
        \centering
        \begin{subfigure}[t]{0.48\linewidth}
            \centering
            \includegraphics[width=\linewidth]{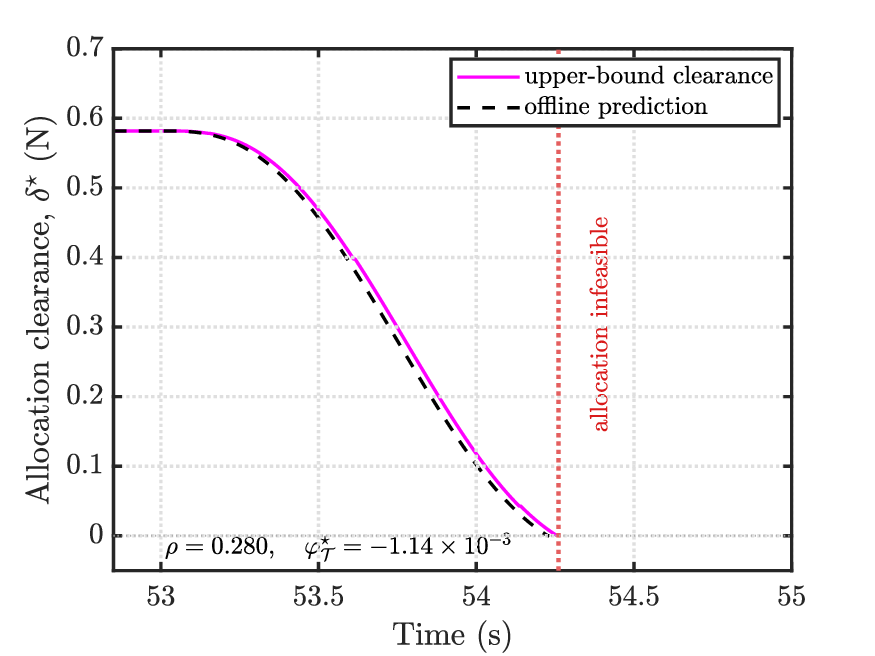}
             \caption{$\mathsf{O}_8$.}
        \end{subfigure}
        \begin{subfigure}[t]{0.48\linewidth}
            \centering
            \includegraphics[width=\linewidth]{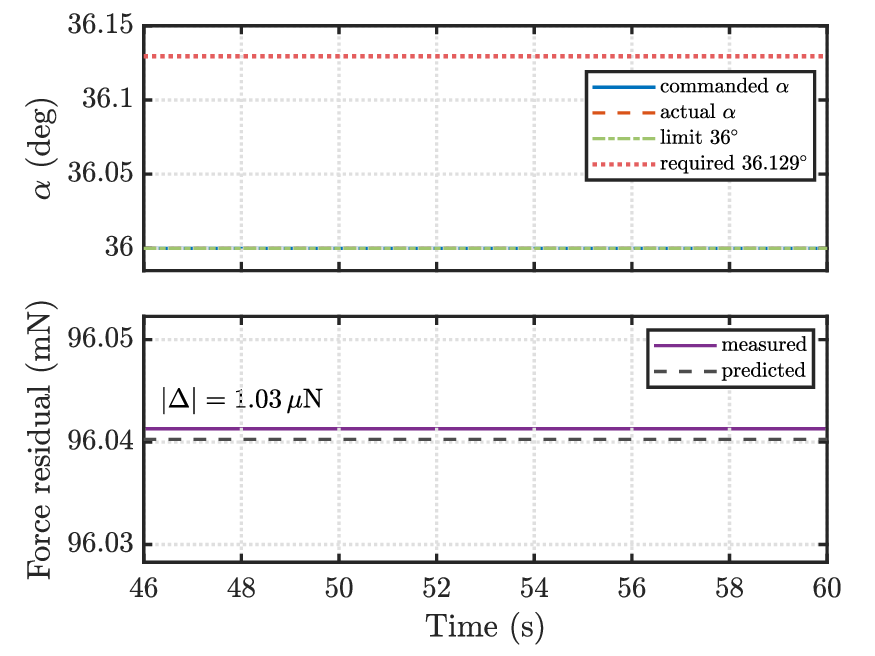}
             \caption{$\mathsf{VTQ}_4$.}
        \end{subfigure}
        \caption{Beyond-boundary tests: (a)~$\mathsf{O}_8$ at $\rho=0.280$ loses feasibility at the predicted thrust facet; (b)~$\mathsf{VTQ}_4$ at $\rho=0.730$ reaches the tilt limit with the predicted force residual.}\vspace{-4mm}
        \label{fig:d6}
    \end{figure} 
    The simulations distinguish reserve exhaustion at $\varphi_{\mathrm{push}}^{\star}=\epsilon_0$ from actuator-limited loss of realizability. The task-specific results and local quadrotor certificate support no architecture-wide superiority or global guarantee. Software-in-the-loop experiments for our study can be viewed at \href{https://youtu.be/M-SZiH8Xie4}{https://youtu.be/M-SZiH8Xie4}.

	\section{Conclusions}
	We proposed a task-relative framework that converts residual wrench authority into actuator-feasible local tracking guarantees for sustained aerial physical interaction. We used certified reserves to synthesize structured gains with nominal exponential convergence and robust ultimate boundedness under the stated assumptions. Task-dependent margins quantify corrective authority that rank and conditioning cannot certify. Equal-thrust comparisons linked morphology to attainable reserves, while redundant allocation improved lower certificates. Contact simulations distinguished exhaustion of a prescribed reserve from actual infeasibility and reproduced the predicted thrust and tilt limits. The results support task-specific morphology and control design. Architectural comparisons are confined to the tested task families, and the variable-tilt certificate remains local.
	\bibliographystyle{IEEEtran}
	\bibliography{references_ACC}
\end{document}